\documentclass[lettersize,journal]{IEEEtran}

\usepackage[utf8]{inputenc}
\usepackage[T1]{fontenc}
\usepackage{microtype}
\usepackage{amsmath, amssymb}

\usepackage{amsthm}
\usepackage{graphicx}
\graphicspath{{figures/}}
\usepackage{booktabs}
\usepackage{xcolor}
\usepackage{colortbl}
\usepackage{enumitem}
\usepackage{multirow}
\usepackage{array}
\usepackage{stfloats}
\usepackage{tikz}
\usetikzlibrary{arrows.meta,positioning,calc,fit,backgrounds}
\usepackage{pgfplots}
\pgfplotsset{compat=1.16}
\usepackage{algorithm}
\usepackage{algpseudocode}
\usepackage{cite}
\usepackage[colorlinks=true,linkcolor=black,citecolor=blue,urlcolor=blue]{hyperref}
\usepackage{cleveref}
\usepackage{orcidlink}   
\crefname{assumption}{Assumption}{Assumptions}
\Crefname{assumption}{Assumption}{Assumptions}
\definecolor{cbblue}{RGB}{31,119,180}
\definecolor{cbred}{RGB}{214,39,40}
\definecolor{cbgreen}{RGB}{44,160,44}
\definecolor{cbgray}{RGB}{120,120,120}

\definecolor{goldc}{HTML}{FFE9A6}
\definecolor{silverc}{HTML}{DEDEDE}
\definecolor{bronzec}{HTML}{EBC9A0}
\newcommand{\rkA}{\cellcolor{goldc}}   
\newcommand{\rkB}{\cellcolor{silverc}} 
\newcommand{\rkC}{\cellcolor{bronzec}} 
\newcommand{\medallegend}{\colorbox{goldc}{1st}\,/\,\colorbox{silverc}{2nd}\,/\,\colorbox{bronzec}{3rd} per split (ties share a rank)}

\newtheorem{theorem}{Theorem}
\newtheorem{proposition}{Proposition}
\newtheorem{corollary}{Corollary}
\newtheorem{assumption}{Assumption}
\newtheorem{remark}{Remark}

\newcommand{\takeaway}[1]{\par\smallskip\noindent{\small\itshape
  Take-away: #1\par}}
\newcommand{\keybox}[1]{\begin{center}\fbox{\parbox{0.955\linewidth}{\small #1}}\end{center}}

\title{CW-BASS v2: Saturation-Aware Pseudo-Label Selection for\\
       Semi-Supervised Segmentation under Foundation-Model Teachers}

\author{Ebenezer~Tarubinga~\orcidlink{0009-0004-7340-1873}%
\thanks{E.~Tarubinga is with the AI Research Department, Ebenworks Systems,
Seoul, Korea.\protect\\ E-mail: research@ebenworks.co\protect\\
Corresponding author.}}

\begin{document}

\maketitle

\begin{abstract}
Semi-supervised semantic segmentation has long turned on one question, which
pseudo-labels to trust, and a generation of selection rules, dynamic thresholds,
per-class curricula, soft confidence weights, answered it for the noisy,
under-confident ResNet teachers of their day. Self-supervised foundation encoders
change the regime: with a DINOv2 teacher, confidence \emph{saturates}, so the
filtering that helped a weak teacher can hurt a strong one. We propose
\textbf{CW-BASS~v2}, a \emph{saturation-aware} pseudo-label selection method that
reads the teacher's confidence regime rather than committing to one rule. It pairs
held-out calibration, an unbiased per-class noise estimate, with a self-adaptive
confidence floor that provably bounds retention away from $1$, and combines them in
a one-pass \emph{gate}: measure the reliability of the teacher's confident set,
$\pi_{\mathrm{kept}}{=}\Pr[\text{correct}\mid c{\ge}\tau]$, on a held-out slice, and
filter strictly when it meets the confidence demanded ($\pi_{\mathrm{kept}}{\ge}\tau$),
falling back to the adaptive floor otherwise. The boundary is the pre-existing
operating threshold, not a value tuned to mIoU, and across six DINOv2 teachers it
makes the correct strict-vs-floor call \emph{blind}. CW-BASS~v2 thus recovers the
UniMatch~V2 operating point on the saturated benchmarks by selecting strict
(Pascal~VOC~1/8 $87.4$ against its reported $87.9$; Cityscapes within $0.5$), and
improves on it where the confident set is \emph{unreliable}
($\pi_{\mathrm{kept}}{\approx}89\%$, ADE20K), where the floor edges ahead ($+1.5$,
single seed). The gate is principled because the failure it avoids is measured, not
assumed: on a reliable, saturated teacher the confidence distribution's dynamic
range collapses ($98\%$ of Pascal pixels ${\ge}0.95$), so an adaptive cutoff floods
the retention mask and self-training decays into confirmation bias.

\end{abstract}


\section{Introduction}

\begin{figure}[!htb]
\centering
\includegraphics[width=0.9\linewidth]{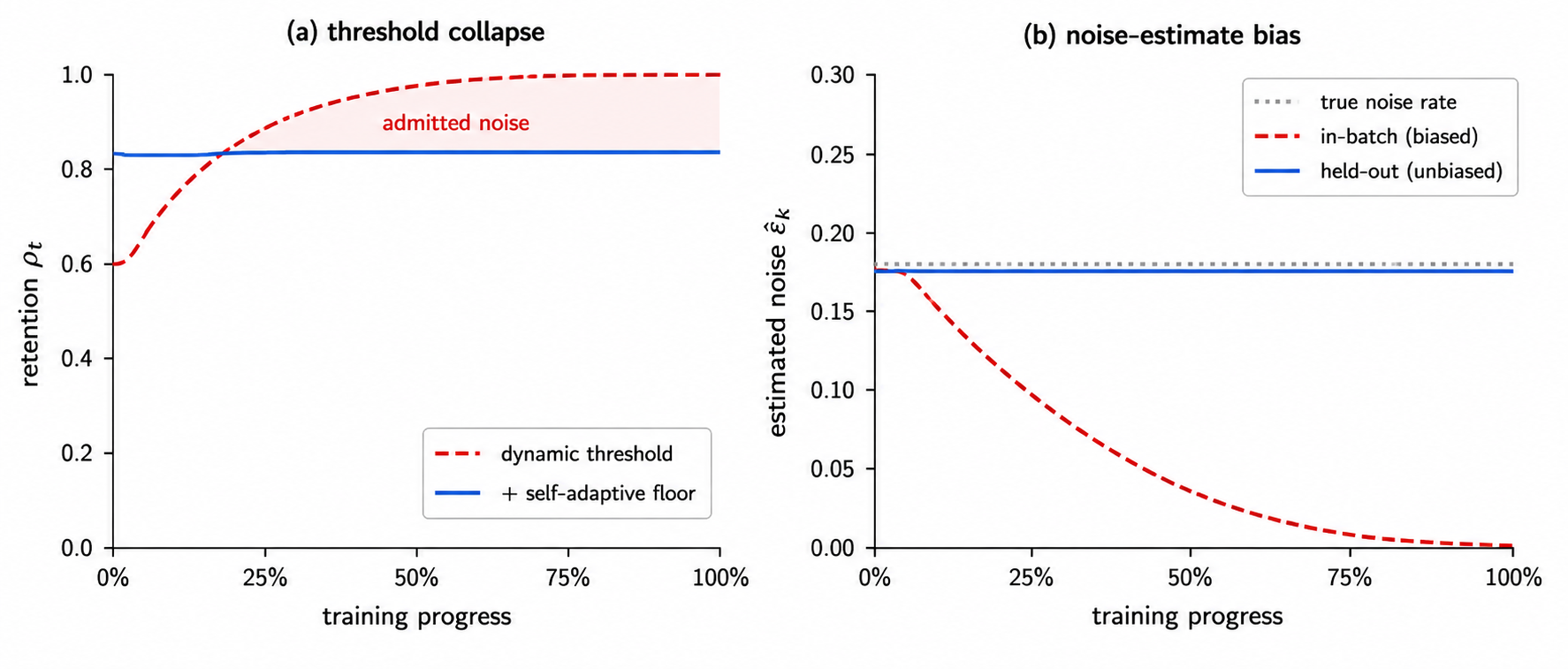}
\caption{\textbf{The mechanism behind CW-BASS~v2's saturation gate.}
\emph{Both panels are schematic}, drawn to show the shape of each effect, not its
measured magnitude; the measured versions are \Cref{fig:maskratio}~(a) and
\Cref{tab:gate}~(b). \textbf{(a)} Under a confident DINOv2 teacher the dynamic
threshold sits near its lower clamp, so retention $\rho_t$ drifts to~$1$ and
training floods with noise (Cor.~\ref{cor:collapse}); the self-adaptive floor pins
$\rho_t$ to a fixed quantile (Thm.~\ref{thm:floor}) --- in the real run it climbs
to that quantile over the first epochs rather than starting there --- but still
does not match a strict threshold on a saturated teacher. \textbf{(b)} An
in-batch noise estimator measures error on pixels the student trains on,
collapsing to~$0$ regardless of quality; our held-out calibration is unbiased
(Prop.~\ref{prop:unbiased}) and breaks that loop. We measure this gap directly:
on the ADE20K teacher, confident-set reliability reads $98.4\%$ in-batch against
$89.3\%$ held out (Sec.~\ref{sec:calibration}).}
\label{fig:teaser}
\end{figure}

\IEEEPARstart{D}{ense} pixel annotation is among the most expensive forms of
supervision in
computer vision: labeling a single Cityscapes image takes roughly
$90$~minutes~\cite{cps}. Semi-supervised semantic segmentation (SSSS) attacks
this cost by learning from a small labeled set together with a large pool of
unlabeled images. The dominant paradigm is \emph{self-training with
consistency regularization}: a teacher assigns pseudo-labels to weakly
augmented unlabeled images, and a student is trained to reproduce them under
strong perturbations~\cite{fixmatch,unimatch}. The central difficulty is that
pseudo-labels are noisy, and a model trained on its own mistakes reinforces
them: confirmation bias~\cite{lee2013pseudo}. Essentially every SSSS method
can be read as a particular answer to one question: \emph{which pseudo-labels
should we trust, and how much?}

\paragraph{The ResNet-era consensus}
For most of the field's history the answer was studied under ResNet-50/101
backbones with DeepLabV3+~\cite{deeplabv3plus} decoders, where teacher
pseudo-labels are noisy enough that the \emph{noise} term dominates the
noise--coverage tradeoff. Methods improved by being smart about
\emph{selection}: confidence thresholding~\cite{fixmatch}, curriculum and
per-class thresholds~\cite{flexmatch,freematch}, soft confidence
weighting~\cite{softmatch}, learning from unreliable
pixels~\cite{u2pl}, and re-distributing biased pseudo-label
priors~\cite{dars,ael}. CW-BASS~\cite{cwbass} belongs to this lineage. It
contributed two mechanisms: (i) a \emph{confidence-weighted} cross-entropy
that scales each pixel's loss by a power of its teacher confidence, softly
down-weighting unreliable pseudo-labels rather than hard-thresholding them;
and (ii) a Sobel-derived \emph{boundary-aware} auxiliary that concentrates
supervision on object edges, where segmentation models are most
uncertain. A per-batch \emph{dynamic threshold} adapted the retention cutoff
to the teacher's mean confidence. On Pascal~VOC and Cityscapes with
ResNet-50, the combination was competitive with the methods of its day.

\paragraph{The foundation-model shift}
Two developments have since changed the regime in which all of these
mechanisms operate. First, self-supervised foundation encoders (DINOv2 in
particular~\cite{dinov2}) produce dense features on which a lightly
fine-tuned decoder~\cite{dpt} reaches pseudo-label accuracy that ResNet
backbones never approached. Second, UniMatch~V2~\cite{unimatchv2} established
that, on top of such a backbone, the strongest results come from
\emph{stricter} filtering (a fixed high threshold $\tau{=}0.95$), two
independent strong views, and complementary channel dropout, not from the
looser, coverage-seeking filtering that the noise-dominated ResNet regime
rewarded. Together they have moved Pascal~VOC 1/8 from the
mid-$70$s~mIoU of the ResNet era to the high-$80$s. This is not a
free lunch handed to every old method: it is a \emph{change of operating
point} that re-prioritises which failure modes matter. A mechanism designed
to squeeze coverage out of a noisy teacher can become inert, or actively
harmful, when the teacher is already accurate.

\paragraph{This paper}
We present CW-BASS~v2, a \emph{saturation-aware} pseudo-label selection method
that carries CW-BASS into the foundation-model regime. Its premise is that no
single threshold rule is right across confidence regimes, the same adaptive
filtering that helps an under-confident teacher hurts a saturated one, so
CW-BASS~v2 does not commit to one rule. It reads the teacher's confidence regime
with a one-pass diagnostic and deploys whichever rule the regime calls for, strict
filtering when the teacher saturates, adaptive filtering when it does not. As a
deployed system it reproduces the UniMatch~V2 operating point on the saturated
benchmarks and \emph{improves} on it where the teacher's confident set is
unreliable (Sec.~\ref{sec:generality}). What makes this design principled rather than a hedge
is a mechanism we measure at every link and that decides the gate: we surface two
failure modes of confidence-adaptive selection, \emph{latent} at ResNet strength
and \emph{dominant} at foundation-model strength (\Cref{fig:teaser}), neither of
which the original CW-BASS could have observed. We report them candidly, including
that CW-BASS~v2's own self-adaptive floor, run \emph{unconditionally} on a
saturated teacher, is the \emph{worst} variant we test ($82.32$ vs.\ strict's
$87.40$), which is exactly why the method does not run it there.

\begin{enumerate}[leftmargin=1.4em,itemsep=2pt]
  \item \textbf{The overconfidence of in-batch noise estimation.} Any method
  that estimates pseudo-label noise from the labeled training pixels (as
  feedback-driven per-class thresholds~\cite{encore} and our calibration-fed
  adaptive scheme do) measures error on data the student has been trained to
  fit. The estimate is biased toward zero and grows more so as training
  proceeds, dragging any noise-driven threshold downward independently of the
  true pseudo-label quality. We make this precise
  (Proposition~\ref{prop:unbiased}) and remove the bias with a held-out
  calibration slice.

  \item \textbf{The collapse of the dynamic threshold.} Under DINOv2 the
  teacher's confidence saturates, and we show that the CW-BASS dynamic
  threshold is upper-bounded by a small constant in this regime. Once nearly
  all confident pixels exceed it, retention drifts to~$1$ and training floods
  with whatever residual noise remains. This collapse manifests only late in
  training and at high teacher confidence; the original ResNet evaluation
  never ran in that regime.
\end{enumerate}

Our contributions are a method and the analysis that makes it principled. We do
not claim a new peak accuracy, UniMatch~V2's strict recipe holds that on the
saturated benchmarks, and we reproduce it, but a selection method that is never
worse than that recipe and \emph{better} where adaptive filtering belongs, together
with the mechanism that tells the two regimes apart:

\begin{enumerate}[leftmargin=1.4em,itemsep=2pt]
  \item \textbf{A saturation-aware selection method with an explicit gate.}
  CW-BASS~v2 measures one statistic on a held-out slice, the reliability of the
  teacher's confident set, $\pi_{\mathrm{kept}}{=}\Pr[\text{correct}\mid
  c{\ge}\tau]$, and deploys strict filtering when it is at least as accurate as the
  confidence demanded ($\pi_{\mathrm{kept}}{\ge}\tau$), its adaptive floor
  otherwise (Sec.~\ref{sec:gate}). The boundary is the operating threshold, not a
  value tuned on mIoU, and across six DINOv2 teachers it makes the correct
  strict-vs-floor call \emph{blind} (\Cref{fig:gate}). The deployed method is therefore never worse
  than the strict state of the art: it selects strict on Pascal~VOC ($87.4$,
  matching UniMatch~V2) and Cityscapes (a near-tie), and improves on it on the one
  teacher whose confident set is unreliable ($\pi_{\mathrm{kept}}{\approx}89\%$,
  ADE20K), where the floor reaches $50.6$~mIoU ($+1.5$ over strict; single seed).

  \item \textbf{An unbiased noise diagnostic: held-out calibration.} We split the
  labeled set into a training slice and a small calibration slice ($\alpha{=}5\%$),
  estimate per-class pseudo-label noise only on the held-out slice, and prove the
  estimator is unbiased where the in-batch estimator (used implicitly by
  feedback-driven per-class thresholds~\cite{encore}) is downward-biased
  (Proposition~\ref{prop:unbiased}). This breaks the feedback loop that drags
  noise-driven thresholds down regardless of true quality, and gives the method an
  honest read of the teacher's regime rather than a coverage-chasing one; its
  prediction, that the estimated noise $\widehat\varepsilon_k$ should \emph{not}
  fall for the classes whose thresholds the adaptive rule lowers, is borne out by
  the per-class IoU evidence (\Cref{tab:tail}).

  \item \textbf{A stability-guaranteed confidence floor.} We introduce a floor
  that scales with the teacher's mean confidence and prove
  (Theorem~\ref{thm:floor}, Corollary~\ref{cor:collapse}) that it pins retention
  to a fixed quantile bounded away from~$1$, whereas the bare dynamic threshold
  provably collapses to full retention. The floor is the adaptive rule the gate
  engages under an \emph{unreliable} teacher (where it is \emph{best}, ADE20K); run
  \emph{unconditionally} on a reliable, saturated teacher it stabilises but does not
  recover accuracy (Pascal, $82.32$ vs.\ strict $87.40$), which is exactly what
  makes the gate necessary rather than optional.

  \item \textbf{The mechanism that decides the gate.} We name the failure chain
  and measure each link on the saturated Pascal teacher: \emph{confidence
  saturation} (the DINOv2 teacher's confidence piles up near~$1$) causes
  \emph{dynamic-range collapse} (the adaptive cutoff degenerates toward a constant,
  pinned in $[0.300,0.331]$), which causes \emph{mask flooding} (retention
  saturates near~$1$), which yields the \emph{early-peak-then-decline} signature as
  confirmation bias~\cite{arazo2020pseudo} and feature
  distortion~\cite{kumar2022finetuning} take over. This explains \emph{why} strict
  wins wherever the confident set is reliable, Pascal and Cityscapes at every
  DINOv2 scale, and, by measuring the six teachers' confidence geometry directly
  (\Cref{tab:gate}), corrects a tempting misreading: the teacher on which adaptive
  wins (ADE20K) is not under-confident but confidently \emph{unreliable}
  ($\pi_{\mathrm{kept}}{\approx}89\%$), the very property the gate reads. This
  doubles as a controlled, batch-matched analysis of adaptive pseudo-label
  thresholding on foundation-model backbones, the first we are aware of.
\end{enumerate}

The mechanism, the matched-batch control, and the trajectory-centred reporting of
Sec.~\ref{sec:negative} also distil into four cheap checks
(Sec.~\ref{sec:discussion}) any practitioner can run before choosing a threshold
rule; the first, the one-pass reliability measurement, is the signal CW-BASS~v2
gates on. A generation of selection mechanisms was designed for weak,
under-confident teachers, a regime foundation backbones leave behind on the
standard benchmarks but not universally; CW-BASS~v2's job is to read, from the
teacher's own calibration, which regime it faces and act accordingly: strict when
the confident set is trustworthy (the common foundation-model case), the adaptive
floor when it is not (Sec.~\ref{sec:generality}). The scientific counterpart is
simple: measure the mechanism before trusting the intuition.

\section{Related Work}
\label{sec:related}

\paragraph{Consistency regularization and self-training for SSSS}
Modern SSSS descends from two threads: consistency regularization, which
enforces stable predictions under perturbation~\cite{meanteacher,cps,psmt},
and self-training, which retrains on the model's own
pseudo-labels~\cite{lee2013pseudo,stplusplus}. FixMatch~\cite{fixmatch}
unified them with weak-to-strong consistency: confident pseudo-labels from a
weak view supervise a strong view. UniMatch~\cite{unimatch} ported this to
segmentation and added feature-space perturbation; AugSeg~\cite{augseg} and
AEL~\cite{ael} refined the augmentation and sampling.
AllSpark~\cite{allspark} reconstructs labeled features from unlabeled ones in
a transformer, and CorrMatch~\cite{corrMatch} propagates labels via
correlation matching. UniMatch~V2~\cite{unimatchv2} is the current state of
the art: a DINOv2 backbone with a strict fixed threshold, dual strong views,
and complementary dropout. CW-BASS~\cite{cwbass} sits in this family, adding
confidence weighting and a boundary auxiliary; \emph{we keep its
weak-to-strong, EMA-teacher, dual-perturbation scaffold and re-examine its
selection mechanisms under the UniMatch~V2 backbone.}

\paragraph{Pseudo-label selection and adaptive thresholds}
Confidence thresholding is the canonical selection rule~\cite{fixmatch}, but a
fixed global threshold trades coverage against noise poorly when classes
differ in difficulty. FlexMatch~\cite{flexmatch} introduced
curriculum pseudo-labeling with per-class thresholds derived from learning
status; FreeMatch~\cite{freematch} replaced the schedule with a
self-adaptive global threshold modulated per class by EMA confidence
statistics; SoftMatch~\cite{softmatch} dropped hard thresholds for a
truncated-Gaussian soft weight; DASH~\cite{dash} grew the threshold
dynamically during training. In segmentation,
CAFS~\cite{cafs} and ENCORE~\cite{encore} adapt per-class thresholds: CAFS from
a held-out labeled calibration, ENCORE from a training-feedback signal.
FARCLUSS~\cite{farcluss} combines several of these levers at once, replacing hard
pseudo-labels with soft top-$K$ class distributions and weighting each pixel by
an entropy-derived reliability score alongside an adaptive class rebalancing
term. Every one of these rules reads the same signal our analysis shows
degenerates under a saturated teacher, an entropy or confidence \emph{spread}
across pixels and classes.
\emph{Our per-class scheme (Sec.~\ref{sec:perclass}) is a faithful instantiation
of precisely this direction}: it sets per-class cutoffs from held-out
calibrated noise estimates, which is the CAFS mechanism, and the in-batch versus
held-out contrast we formalise (Proposition~\ref{prop:unbiased}) is the ENCORE
training-feedback signal made unbiased, so the audit tests these published
methods' \emph{mechanism}, not a strawman. We did not re-run their released code
and flag this as a limitation (Sec.~\ref{sec:limitations}). Our self-adaptive
floor is closest in spirit to FreeMatch's confidence-tracking threshold, but plays
a different role: rather than \emph{being} the threshold, it \emph{lower-bounds}
any dynamic threshold, and we prove this prevents the retention collapse that
FreeMatch-style rules do not guard against in the high-confidence regime.
\emph{Our analysis (Sec.~\ref{sec:negative}) shows that the per-class direction
these methods pursue does not transfer to reliable, saturated foundation-model
teachers, though it retains value where the confident set is unreliable, which is
what CW-BASS~v2 gates on.}

Audits exist in adjacent territory: Oliver \emph{et
al.}~\cite{oliver2018realistic} audited the
evaluation practice of SSL classification, and
Landgraf \emph{et al.}~\cite{landgraf2025rethinking} recently audited
UniMatch~V2's reliability and
robustness. Neither isolates the threshold rule as the single varied factor;
to our knowledge, ours is the first batch-matched, mechanism-level audit of
adaptive pseudo-label thresholding on a foundation-model backbone.

\paragraph{Foundation backbones for dense prediction}
DINOv2~\cite{dinov2} learns dense visual features without labels; a frozen or
lightly fine-tuned DINOv2 encoder with a DPT-style decoder~\cite{dpt} is a
strong dense predictor, and vision-language guidance (SemiVL~\cite{semivl})
pushes SSSS further. The shift to such backbones is the premise of our audit:
the noise statistics that ResNet-era selection mechanisms were tuned for no
longer hold, so mechanisms must be re-derived rather than ported.

\paragraph{Calibration and learning with noisy labels}
Estimating and correcting label noise is classical~\cite{natarajan2013noisy};
confident learning~\cite{confidentlearning} estimates noise transition
matrices from held-out predictions, and our held-out calibration is a
segmentation-specific, online instance of the same de-biasing principle. The
closest segmentation precedent is CAFS~\cite{cafs}, which also calibrates
per-class statistics on a held-out slice of the labeled set; the distinction is
one of \emph{purpose}. CAFS uses that calibration to \emph{set} per-class
adaptive thresholds (precisely the coverage-seeking direction our audit finds
counter-productive once the teacher saturates, Sec.~\ref{sec:negative}), whereas we
repurpose held-out calibration as an \emph{unbiased diagnostic}
(Proposition~\ref{prop:unbiased}) whose role is to explain \emph{why} that
direction fails, not to chase coverage.
Probability calibration~\cite{guo2017cal,betacal,dirichletcal} and
long-tailed calibration~\cite{zhong2021cal} motivate why raw softmax
confidence is an unreliable noise proxy; we use confidence only after held-out
re-estimation. Our treatment of held-out per-class error rates as an unbiased
control signal (rather than a bounded guarantee) is in the spirit of the
noisy-label and self-training theory of Wei \emph{et al.}~\cite{wei2021theoretical}.

\paragraph{Class imbalance}
Long-tailed losses~\cite{ldam,seesaw} and class-rebalancing self-training
(CReST~\cite{crest}, imbalanced SSL~\cite{guo2022class}) motivate the
rarity-scaled coverage penalty we test in the per-class scheme. Our negative
result indicates that, at foundation-model strength, rebalancing the
\emph{threshold} is the wrong lever, even if rebalancing the \emph{loss}
remains useful.

\paragraph{The regime shift, quantified}
The premise of our analysis is a change of regime. The jump from ResNet/SegFormer
methods in the high-$70$s to DINOv2-based UniMatch~V2 at $87.9$ on Pascal~VOC 1/8
(the full SSSS landscape is in \Cref{tab:landscape}, Sec.~\ref{sec:sota}) is
driven primarily by the backbone, not by selection cleverness. A
threshold-side mechanism tuned for the noise statistics of the upper rows need
not transfer to the bottom row, which is exactly what we find. The accuracy
jump is only half the story; the other half is a change in the teacher's
\emph{confidence statistics}. \Cref{fig:calibration} contrasts the two regimes:
a ResNet-50 teacher spreads its pixel confidence across $[0.3,1]$, giving a
threshold a wide operating range, whereas the DINOv2 teacher collapses almost
all of its mass into $[0.95,1]$ ($98\%$ of pixels exceed $0.95$, versus $53\%$
for the ResNet). This is \emph{not} an aggregate-miscalibration artefact: the
DINOv2 teacher has the lower expected calibration error of the two
($0.007$ vs.\ $0.024$), precisely because where it is confident it is
usually right. The problem is the opposite: \emph{confidence saturation},
and with it \emph{dynamic-range collapse}: the distribution has so little
spread that a confidence-adaptive threshold has almost nothing left to read.
When the signal a rule depends on degenerates to a spike at~$1$, the rule
cannot discriminate, however well-calibrated that spike is.

\begin{figure}[!htb]
\centering
\IfFileExists{figures/calibration.pdf}{%
  \includegraphics[width=0.9\linewidth]{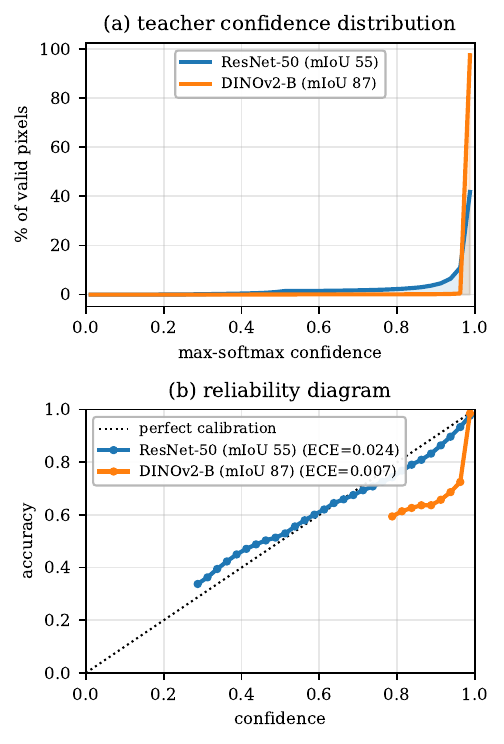}%
}{\fbox{\parbox{0.9\linewidth}{\centering\color{gray}\vspace{1.8cm}
  \textbf{Figure placeholder} — \texttt{scripts/plot\_calibration.py}
  (ResNet-50 vs DINOv2 teacher).\vspace{1.8cm}}}}
\caption{\textbf{The confidence-statistics half of the regime shift.} (Pascal~VOC
1/8 val pixels, EMA teachers; ResNet-50 is a representative partially-trained
model.) \textbf{(a)} The ResNet-50 teacher spreads its max-softmax confidence across the
range, whereas the DINOv2 teacher piles $98\%$ of pixels into $[0.95,1]$ ($53\%$
for ResNet). \textbf{(b)} Reliability. Read the two panels together: the DINOv2
teacher's \emph{aggregate} ECE is the lower of the two ($0.007$ vs.\ $0.024$)
because almost all of its mass sits in the top bin, where it is accurate --- but
in the sparsely-populated band below $0.95$ it is markedly \emph{over}-confident,
its curve falling well under the diagonal where the ResNet's tracks it. Both
facts matter, and they point the same way: the confident mass is trustworthy
(hence strict works), while the thin low-confidence band a relaxed cutoff would
recover is exactly where the teacher is wrong (hence adaptive does not). The
obstacle for adaptive thresholding is therefore not aggregate miscalibration but
\emph{dynamic-range collapse}: almost no spread left to threshold on, and what
spread remains is error-enriched.}
\label{fig:calibration}
\end{figure}

\section{Preliminaries}
\label{sec:prelim}

\begin{table}[t]
\centering
\caption{\textbf{Notation.}}
\label{tab:notation}
\small
\begin{tabular}{ll}
\toprule
Symbol & Meaning \\
\midrule
$\mathcal{L},\mathcal{U}$ & labeled / unlabeled sets \\
$\mathcal{L}_{\mathrm{tr}},\mathcal{L}_{\mathrm{cal}}$ & training / calibration slices of $\mathcal{L}$ \\
$f_\theta, f_{\theta'}$ & student / EMA teacher \\
$\hat y_{h,w}, c_{h,w}$ & pseudo-label and confidence at pixel $(h,w)$ \\
$\tau_k$ & retention threshold for class $k$ \\
$\bar c_t$ & EMA of mean teacher confidence \\
$\mu_k$ & running mean confidence of class $k$ \\
$\widehat\varepsilon_k(\tau)$ & estimated per-class noise rate at cutoff $\tau$ \\
$d_t$ & EMA teacher decay at step $t$ \\
$\gamma,\,\beta_b$ & confidence-weighting exponent, boundary weight \\
$\tau_0,\,\beta$ & base threshold and slope of the dynamic rule \\
$\rho_t$ & retention (mask) ratio at step $t$ \\
$\pi_{\mathrm{kept}}$ & reliability of the confident set, $\Pr[\text{correct}\mid c{\ge}\tau]$ \\
$s,\,\alpha$ & floor scale, calibration fraction \\
$K$ & number of classes \\
\bottomrule
\end{tabular}
\end{table}

\paragraph{Problem setup}
(\Cref{tab:notation} summarises notation.)
We are given a labeled set $\mathcal{L}=\{(x_i,y_i)\}_{i=1}^{N_L}$ with
dense labels $y_i\in\{1,\dots,K\}^{H\times W}$ and a larger unlabeled set
$\mathcal{U}=\{x_j\}_{j=1}^{N_U}$, $N_U\gg N_L$. We train a segmentation
network $f_\theta$ (encoder + decoder) and maintain an exponential
moving-average (EMA) teacher $f_{\theta'}$ with
$\theta'_{t+1}=d_t\theta'_t+(1-d_t)\theta_t$ and ramp-up
$d_t=\min(1-1/(t{+}1),\,0.996)$. (We write the EMA decay $d_t$ to keep $\gamma$
for CW-BASS's confidence-weighting exponent, \Cref{eq:cwce}.)

\paragraph{Weak-to-strong consistency}
For an unlabeled image $x$, the teacher produces per-pixel class posteriors
$p_{\theta'}(x)\in\Delta^{K-1}$ on a weakly augmented view; the pseudo-label
and confidence at pixel $(h,w)$ are
\begin{equation}
\hat y_{h,w}=\arg\max_k\, p_{\theta'}(x)_{h,w,k},\qquad
c_{h,w}=\max_k\, p_{\theta'}(x)_{h,w,k}.
\label{eq:pseudo}
\end{equation}
A retention mask $\mathcal{M}=\{(h,w):c_{h,w}\ge\tau_{\hat y_{h,w}}\}$ selects
pseudo-labels above a (possibly class-dependent) threshold $\tau_k$, and the
student is supervised on strongly perturbed views to match $\hat y$ on
$\mathcal{M}$. Following UniMatch~\cite{unimatch}, we use two strong streams:
an image-space stream (strong photometric augmentation + CutMix) and a
feature-perturbation stream (channel dropout), sharing one backbone forward.
The total objective is
\begin{equation}
\mathcal{L}=\tfrac{1}{2}\Bigl(\mathcal{L}_x
+\tfrac{1}{2}(\mathcal{L}_s+\mathcal{L}_{\mathrm{fp}})\Bigr),
\label{eq:total}
\end{equation}
with $\mathcal{L}_x$ the supervised cross-entropy on $\mathcal{L}$ and
$\mathcal{L}_s,\mathcal{L}_{\mathrm{fp}}$ the unlabeled losses on the two
strong streams.

\paragraph{CW-BASS recap}
CW-BASS~\cite{cwbass} instantiates the unlabeled loss with two mechanisms.
(i) A \emph{confidence-weighted} cross-entropy weights each retained pixel by
$c_{h,w}^{\gamma}$,
\begin{equation}
\mathcal{L}_{\mathrm{cw}}=\frac{1}{|\mathcal{M}|}\sum_{(h,w)\in\mathcal{M}}
c_{h,w}^{\gamma}\,\mathrm{CE}\bigl(f_\theta(x)_{h,w},\,\hat y_{h,w}\bigr),
\label{eq:cwce}
\end{equation}
so confident pseudo-labels contribute more gradient ($\gamma$ controls the
sharpness). (ii) A \emph{boundary-aware} term adds cross-entropy restricted
to a Sobel-detected boundary mask $B$ of the pseudo-label map,
$\mathcal{L}_s=\mathcal{L}_{\mathrm{cw}}+\beta_b\,\mathbb{E}_{(h,w)\in B}
[\mathrm{CE}]$, emphasising object edges. (iii) A \emph{dynamic threshold}
adapts the global cutoff to the teacher's mean confidence
$\bar c=\mathbb{E}[c_{h,w}]$,
\begin{equation}
\tau^{\mathrm{dyn}}=\mathrm{clip}\!\left(\frac{\tau_0}{1+e^{-\beta(\bar c-1/2)}},\,
\tau_{\min},\,\tau_{\max}\right),
\label{eq:dyn}
\end{equation}
with $\tau_0{=}0.6$, $\beta{=}0.5$, $\tau_{\min}{=}0.3$, $\tau_{\max}{=}0.95$
in the original work. Because the sigmoid is bounded, \Cref{eq:dyn} can never
exceed $\tau_0\,\sigma(\beta/2)\approx0.34$ for \emph{any} teacher, a ceiling
fixed by these constants that becomes the crux of the collapse analysis
in~Sec.~\ref{sec:floor}.

\takeaway{The dynamic threshold~\eqref{eq:dyn} reads one scalar (the
teacher's mean confidence $\bar c$) through a bounded sigmoid. Everything
that follows turns on what happens to that rule when the signal it reads
saturates.}

\section{Method}
\label{sec:method}

\begin{figure*}[!tb]
\centering
\includegraphics[width=0.8\textwidth]{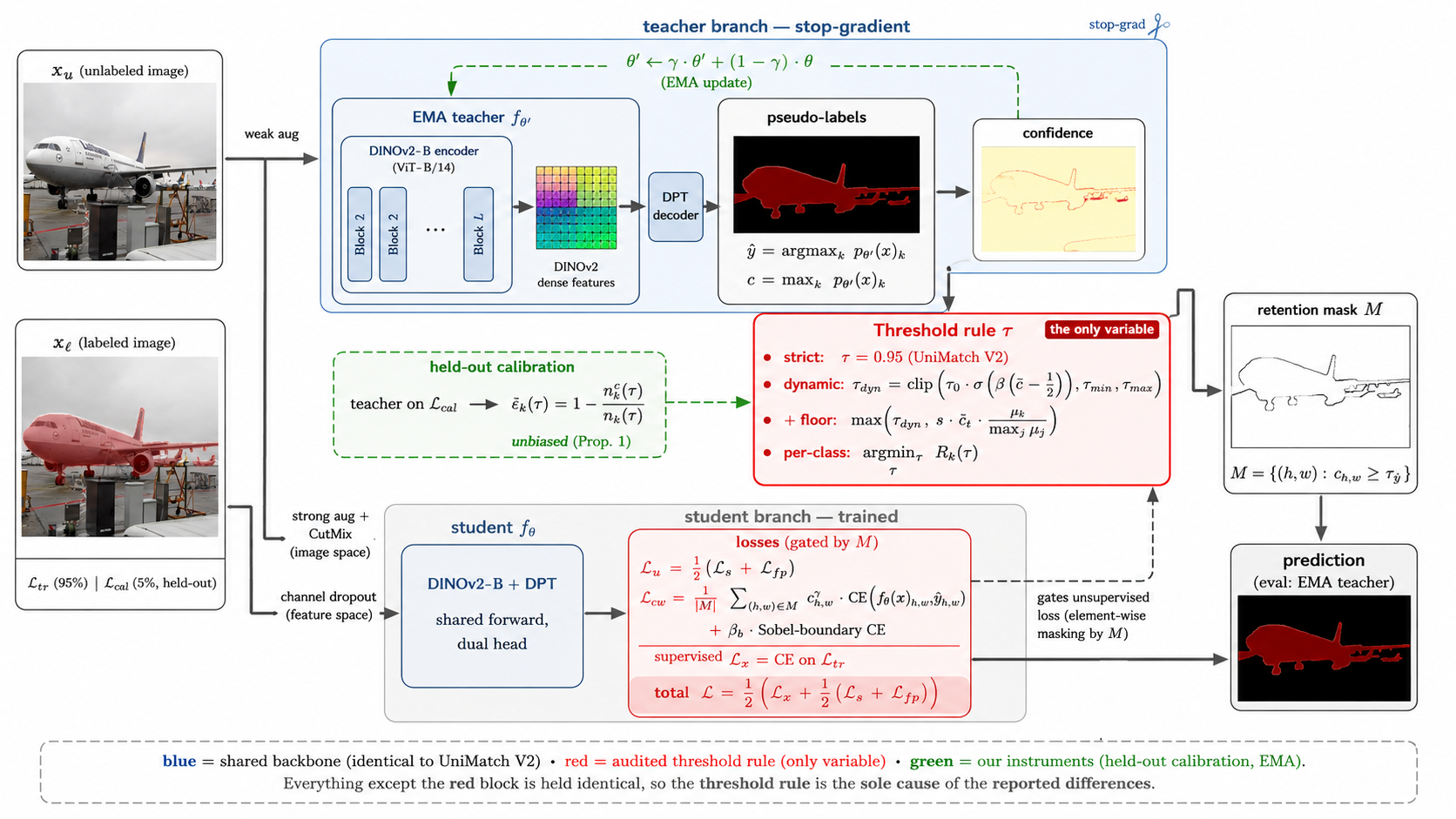}
\caption{\textbf{The CW-BASS~v2 pipeline.} A weak view of each unlabeled image
$x_u$ passes through the \textcolor{cbblue!60!black}{EMA teacher} (DINOv2 encoder
+ DPT decoder) to produce pseudo-labels $\hat y$ and confidences $c$. The
\textcolor{cbred!75!black}{\textbf{threshold rule}~$\tau$} (red, the \emph{gated}
component) turns $c$ into a retention mask $\mathcal{M}$; the four rules shown are
strict $\tau{=}0.95$, the dynamic global threshold, that threshold floored by the
self-adaptive rule, and a per-class adaptive threshold. Two strong views (CutMix;
channel dropout) drive the \textcolor{cbblue!60!black}{student}; the unlabeled
loss on $\mathcal{M}$ is confidence-weighted CE plus a Sobel-boundary term. The
\textcolor{cbgreen!50!black}{green} held-out calibration estimates unbiased
per-class noise $\widehat\varepsilon_k$ (Prop.~\ref{prop:unbiased}), a diagnostic
feeding the gate. All but the red block matches UniMatch~V2.}
\label{fig:arch}
\end{figure*}

CW-BASS~v2 keeps the scaffold of~\Cref{eq:total} and the two original CW-BASS
mechanisms (\Cref{eq:cwce} and the boundary term), and adds three constructs that
make pseudo-label selection \emph{saturation-aware}: held-out calibration
(Sec.~\ref{sec:calibration}), which reads pseudo-label noise without the
downward bias of in-batch estimates; a self-adaptive confidence floor
(Sec.~\ref{sec:floor}), the adaptive rule the method engages under an
unreliable teacher; and a one-pass \emph{gate} (Sec.~\ref{sec:gate}) that selects
between that floor and a strict fixed threshold by measuring the reliability of
the teacher's confident set, $\pi_{\mathrm{kept}}$, on the held-out slice.
\Cref{fig:arch} shows how they attach to
the pipeline. Sec.~\ref{sec:perclass} describes the more ambitious per-class
adaptive scheme we also investigated; Sec.~\ref{sec:negative} measures which rule
wins in which regime, and why, the evidence the gate acts on. Each mechanism is
independently toggleable, enabling the ablation in Sec.~\ref{sec:ablation}.

\subsection{Held-Out Calibration}
\label{sec:calibration}

A method that adapts a threshold to pseudo-label noise needs an estimate of
that noise. The natural estimate (used implicitly by the CW-BASS dynamic
threshold and explicitly by feedback-driven per-class
methods~\cite{encore}) is computed on the labeled \emph{training} pixels: for
class~$k$ at cutoff~$\tau$, count predicted-as-$k$ pixels with confidence
$\ge\tau$ and the fraction that match the ground truth. The problem is that
the student is \emph{trained} on exactly these pixels, so the estimate
measures training error, not generalisation error, and is biased toward zero.

\paragraph{Construction}
We partition the labeled set
$\mathcal{L}=\mathcal{L}_{\mathrm{tr}}\sqcup\mathcal{L}_{\mathrm{cal}}$ with
$|\mathcal{L}_{\mathrm{cal}}|=\alpha|\mathcal{L}|$, $\alpha{=}0.05$. The
supervised gradient $\mathcal{L}_x$ is computed \emph{only} on
$\mathcal{L}_{\mathrm{tr}}$; $\mathcal{L}_{\mathrm{cal}}$ never enters the
supervised loss. Once per epoch we run the EMA teacher over
$\mathcal{L}_{\mathrm{cal}}$ and populate per-class noise rates
\begin{equation}
\widehat\varepsilon_k(\tau)=1-\frac{n_k^{c}(\tau)}{n_k(\tau)},
\label{eq:eps-hat}
\end{equation}
where $n_k(\tau)$ counts calibration pixels predicted as $k$ with confidence
$\ge\tau$ and $n_k^{c}(\tau)$ the subset matching ground truth. We use
$\widehat\varepsilon_k(\tau)$ as an \emph{unbiased point estimate} and control
signal, and deliberately attach no high-probability confidence interval to it:
pixels within an image are strongly spatially correlated, so the effective sample
size is closer to a count of \emph{images} than of pixels (on Pascal 1/8 the
calibration slice is only $\alpha|\mathcal{L}|\!\approx\!9$ images), and a
pixel-count concentration bound would claim a confidence it has not earned. No
claim in the paper depends on such an interval; the estimator's role is to give an
unbiased read of per-class noise (Proposition~\ref{prop:unbiased}), and, in the
gate, of the aggregate reliability $\pi_{\mathrm{kept}}$ (Sec.~\ref{sec:gate}). The
per-class noise histogram is reset at each epoch boundary so the estimate reflects
the \emph{current} teacher rather than stale statistics accumulated early in
training.

\paragraph{Why held-out estimation is necessary}
We formalise the bias that motivates the construction. Fix the teacher
parameters $\theta'$ and a class $k$ and cutoff $\tau$. Let
$\varepsilon_k(\tau)=\Pr[\hat y\neq y \mid \hat y=k,\,c\ge\tau]$ be the true
per-class noise rate on the data distribution.

\begin{proposition}[Held-out de-biasing]
\label{prop:unbiased}
The calibration estimator $\widehat\varepsilon_k^{\,\mathrm{cal}}(\tau)$ of
\Cref{eq:eps-hat}, computed on $\mathcal{L}_{\mathrm{cal}}$ which is disjoint
from $\mathcal{L}_{\mathrm{tr}}$ and unused by the supervised gradient, is
\emph{conditionally unbiased} given the teacher $\theta'$ and the selected
count $n_k(\tau){>}0$:
$\mathbb{E}\bigl[\widehat\varepsilon_k^{\,\mathrm{cal}}(\tau)
\mid \theta',\,n_k(\tau)\bigr]=\varepsilon_k(\tau)$ (the conditioning is needed
because \Cref{eq:eps-hat} is a ratio of random counts; given the selected set
it is a mean of i.i.d.\ error indicators). The in-batch estimator
$\widehat\varepsilon_k^{\,\mathrm{in}}(\tau)$ computed on
$\mathcal{L}_{\mathrm{tr}}$ satisfies
$\mathbb{E}\bigl[\widehat\varepsilon_k^{\,\mathrm{in}}(\tau)\mid\theta',\,n_k(\tau)\bigr]
\le\varepsilon_k(\tau)$, with the gap equal to the teacher's train--population
generalisation gap on confident class-$k$ pixels.
\end{proposition}

\begin{proof}[Proof sketch]
Conditioned on $\theta'$ and the selected set $\{c\ge\tau,\,\hat y=k\}$ of size
$n_k(\tau)$, the calibration pixels are drawn from the same distribution as the
population but are statistically independent of the optimisation of $\theta$
(they never contribute a gradient), so the empirical error
frequency~\eqref{eq:eps-hat} is a conditional mean of i.i.d.\ Bernoulli error
indicators: an unbiased Monte-Carlo estimate of $\varepsilon_k(\tau)$. (The
conditioning matters: \Cref{eq:eps-hat} is a ratio of random counts, unbiased
for $\varepsilon_k$ given the selection event, not marginally.) For the
in-batch estimator, the teacher and student are
coupled to $\mathcal{L}_{\mathrm{tr}}$ through training; a model fit to those
pixels has empirical error no larger than its population error in expectation
(the standard optimism of the training error), and the inequality is strict
whenever the network has capacity to memorise, which segmentation networks
do. The difference is exactly the generalisation gap restricted to confident
class-$k$ predictions. \end{proof}

\Cref{prop:unbiased} explains a concrete failure: under an in-batch estimate,
$\widehat\varepsilon_k^{\,\mathrm{in}}\!\to\!0$ as the student memorises
$\mathcal{L}_{\mathrm{tr}}$, so any threshold that loosens as estimated noise
falls is driven downward \emph{regardless of true pseudo-label quality}. The
held-out construction breaks this feedback loop by construction. This bias is
not hypothetical, and it is exactly large enough to matter for the gate: on the
ADE20K teacher the confident-set reliability measured \emph{in-batch} (on the
labeled-train pixels the student fits) is $\pi_{\mathrm{kept}}^{\mathrm{in}}{=}98.4\%$,
indistinguishable from the reliable Pascal/Cityscapes teachers, whereas measured on
\emph{held-out} data it is only $89.3\%$ (Sec.~\ref{sec:gate}). A gate reading the
in-batch value would wrongly judge this teacher reliable and select strict,
forfeiting the floor's win; the $9$-point optimism is what makes held-out
measurement necessary rather than cosmetic. Calibration is thus a
\emph{diagnostic and a control signal}, not a leaderboard trick.

\subsection{Self-Adaptive Confidence Floor}
\label{sec:floor}

The second failure mode is structural to~\Cref{eq:dyn}. We first show the bare
dynamic threshold collapses under a confident teacher, then introduce a floor
and prove it does not.

\paragraph{The collapse of the dynamic threshold}
For $\bar c\in[0,1]$ the sigmoid in~\Cref{eq:dyn} is increasing but bounded:
$\tau^{\mathrm{dyn}}\le \tau_0\,\sigma(\beta/2)$. With the original
$\tau_0{=}0.6,\beta{=}0.5$ this ceiling is $\approx 0.34$; the unclipped sigmoid
spans only $[0.263,0.337]$ over the entire range $\bar c\in[0,1]$, and after the
lower clamp $\tau_{\min}{=}0.3$ of \Cref{eq:dyn} the operative cutoff lives in
$[0.300,0.337]$, effectively a constant pinned at its lower clamp. This is harmless
at ResNet strength, where confident pixels are scarce and many sit below
$0.34$. But under DINOv2 the confidence distribution concentrates near~$1$,
so almost every pixel clears a $0.34$ cutoff and the retention rate
$\rho=\Pr[c\ge\tau^{\mathrm{dyn}}]\!\to\!1$. Training then optimises against
\emph{every} pseudo-label, noise included. We make this precise in
Corollary~\ref{cor:collapse}.

\paragraph{The floor}
Let $\bar c_t$ be the EMA of the unlabeled-view mean teacher confidence,
$\bar c_{t+1}=m\,\bar c_t+(1-m)\,\mathbb{E}_{x\in\mathcal{B}_u}[c(x)]$, and let
$\mu_k$ be the running per-class mean confidence. We define a per-class floor
\begin{equation}
\tau_k^{\mathrm{floor}}=s\cdot\bar c_t\cdot\frac{\mu_k}{\max_j\mu_j},
\qquad s\in(0,1],
\label{eq:floor}
\end{equation}
and lower-bound the operative threshold by it:
$\tau_k^{\mathrm{final}}=\max(\tau_k^{\mathrm{dyn}},\tau_k^{\mathrm{floor}})$,
with retention mask $\mathcal{M}=\{(h,w):c_{h,w}\ge\tau_{\hat y_{h,w}}^{\mathrm{final}}\}$.
The floor \emph{scales with} $\bar c_t$: as the teacher grows confident the
floor rises with it, so a fixed fraction of the (now higher) confidence mass is
always filtered. The per-class factor $\mu_k/\max_j\mu_j$ keeps the floor
lower for under-learned classes so their scarcer, lower-confidence
pseudo-labels are not over-filtered.

To analyse retention we adopt an explicit, falsifiable model of how the
confidence distribution moves during training.

\begin{assumption}[Scale-family confidence]
\label{ass:scale}
For unlabeled pixels with teacher-predicted class $k$, the confidence at
training time $t$ satisfies $c \stackrel{d}{=}\bar c_t\, V_k$, where $V_k\ge0$
is a time-invariant ``relative confidence'' with CDF $G_k$ and mean $\nu_k$,
and $\bar c_t$ is the overall mean confidence. (Equivalently: as training
proceeds the per-class confidence distribution rescales with the overall
confidence level but keeps its shape.)
\end{assumption}

Under \Cref{ass:scale}, $\mu_k=\mathbb{E}[c\mid k]=\bar c_t\,\nu_k$, so the
ratio $r_k:=\mu_k/\max_j\mu_j=\nu_k/\max_j\nu_j$ is time-invariant.

\begin{theorem}[Bounded, confidence-invariant retention]
\label{thm:floor}
Under \Cref{ass:scale}, the class-$k$ retention rate induced by the floor,
$\rho_k^{\mathrm{floor}}:=\Pr[c\ge\tau_k^{\mathrm{floor}}]$, is independent of
the overall confidence level $\bar c_t$:
\begin{equation}
\rho_k^{\mathrm{floor}}
=\Pr\!\bigl[\bar c_t V_k\ge s\,\bar c_t\, r_k\bigr]
=\Pr\bigl[V_k\ge s\,r_k\bigr]
=1-G_k(s\,r_k).
\label{eq:thm}
\end{equation}
Consequently $\rho_k^{\mathrm{floor}}<1$ whenever $G_k(s\,r_k)>0$ (i.e. there
is positive confidence mass below $s\,r_k$), and the actual retention obeys
$\Pr[c\ge\tau_k^{\mathrm{final}}]\le\rho_k^{\mathrm{floor}}$. As
$\bar c_t\to1$, retention does \emph{not} converge to~$1$; it remains pinned
at the fixed quantile $1-G_k(s\,r_k)$, controlled by~$s$.
\end{theorem}

\begin{proof}
The $\bar c_t$ factors cancel inside the probability in~\Cref{eq:thm},
leaving a statement about the time-invariant $V_k$; this is the claimed
$\bar c_t$-invariance. Positivity of $G_k(s r_k)$ gives
$\rho_k^{\mathrm{floor}}=1-G_k(s r_k)<1$. Finally
$\tau_k^{\mathrm{final}}\ge\tau_k^{\mathrm{floor}}$ implies
$\{c\ge\tau_k^{\mathrm{final}}\}\subseteq\{c\ge\tau_k^{\mathrm{floor}}\}$, so
the actual retention is no larger.
\end{proof}

\begin{corollary}[Collapse of the bare dynamic threshold]
\label{cor:collapse}
Under \Cref{ass:scale}, write $\bar\tau:=\tau_0\,\sigma(\beta/2)$ for the
constant ceiling that upper-bounds the dynamic threshold at every $\bar c_t$
(\Cref{eq:dyn}). The class-$k$ retention is
$\rho_k^{\mathrm{dyn}}=\Pr[\bar c_t V_k\ge\tau^{\mathrm{dyn}}]
=1-G_k(\tau^{\mathrm{dyn}}/\bar c_t)$, and since
$\tau^{\mathrm{dyn}}\le\bar\tau$,
\begin{equation}
\rho_k^{\mathrm{dyn}}\;\ge\;1-G_k\!\bigl(\bar\tau/\bar c_t\bigr)
\;\xrightarrow[\;\bar c_t\to1\;]{}\;1-G_k(\bar\tau).
\label{eq:dyn-limit}
\end{equation}
The contrast with \Cref{thm:floor} is the point: the floored rule holds a
\emph{fixed} quantile of $V_k$, whereas the bare rule's retention floor is
governed by the fixed cutoff ceiling $\bar\tau$ and rises toward
$1-G_k(\bar\tau)$ as the teacher grows confident. That limit equals~$1$
\emph{exactly when} the relative-confidence law $G_k$ places negligible mass
below $\bar\tau$, i.e.\ when the teacher saturates so that almost every
confident pixel clears the ceiling. This last step is a \emph{condition on the
data}, not a consequence of \Cref{ass:scale} (which fixes the shape of $G_k$):
Sec.~\ref{sec:anatomy} verifies it holds empirically ($98\%$ of pixels exceed
$0.95\gg\bar\tau\approx0.34$, so $G_k(\bar\tau)\approx0$), at which point
$\rho_k^{\mathrm{dyn}}\to1$ and the mask floods.
\end{corollary}

\Cref{thm:floor} and \Cref{cor:collapse} formalise the central design point:
because the floor scales with $\bar c_t$ and the bare threshold does not, the
floor converts a quantity that collapses to full retention into one that holds a
fixed, controllable quantile. We are deliberately modest about what this buys.
Once Assumption~\ref{ass:scale} is granted the $\bar c_t$ factors cancel and the
result is elementary; the substance is entirely whether the assumption holds,
which we treat as an empirical question answered by the flat mask-ratio trajectory
(\Cref{fig:trajectory}b, Remark~\ref{rem:scale}), not as a proof. We therefore
read \Cref{thm:floor} as an \emph{empirical-stability} statement, and we state
plainly that it has \emph{no bearing on accuracy}: with $s{=}0.95$ under a
saturated teacher the held quantile still admits $\approx0.91$ of the mass
(Sec.~\ref{sec:anatomy}, \Cref{fig:mechanism}b), so the floor buys a retention that
no longer drifts to~$1$, and \emph{nothing more}: indeed the floored rule is among
the \emph{worst} on mIoU (\Cref{tab:negative}). Its role in the paper is diagnostic
(it isolates the collapse mechanism by removing it), not a route to higher
accuracy.

\paragraph{Corollary~\ref{cor:collapse}, measured: what is, and is not, a
prediction.} Two things must be kept apart, because the distinction decides how
much the data can be said to confirm. The ceiling
$\bar\tau=\tau_0\sigma(\beta/2)\approx0.34$ is \emph{not} an empirical
discovery: it is fixed algebraically by the original constants
$\tau_0{=}0.6,\beta{=}0.5$, so \Cref{eq:dyn} cannot exceed $0.337$ for
\emph{any} backbone, foundation or otherwise. That the measured class-averaged
cutoff sits in $[0.300,0.331]$ over the entire matched-batch~$16$ run therefore
confirms only the algebra and that the sigmoid is saturated (its input
$\beta(\bar c-\tfrac12)$ pinned), not a law about foundation models. The
\emph{empirical} content of \Cref{cor:collapse} is downstream and genuinely
data-dependent: with the cutoff stuck near $0.34$ while the teacher's mean
confidence climbs past $0.88$, the condition $G_k(\bar\tau)\approx0$ holds
($98\%$ of pixels exceed $0.95$), so by~\eqref{eq:dyn-limit} retention floods
to~$1$ (Sec.~\ref{sec:anatomy}, \Cref{tab:retention}); \emph{that} flooding is the
prediction the run bears out. The ceiling is thus a property of the rule's
hyperparameters, not of the regime: raising $\tau_0$ or $\beta$, or lifting the
lower clamp $\tau_{\min}$, raises it, and in the limit $\tau_{\min}\!\to\!0.95$
the dynamic rule degenerates into the strict baseline itself, so the
interesting question is not whether \emph{these} constants collapse (they must,
by construction) but whether \emph{any} adapt-downward rule can beat strict once
the confidence range has collapsed; Sec.~\ref{sec:negative} takes that up
directly. By contrast, the floored rule's operative threshold (and the
per-class rule's, which rises to $0.805$) does track confidence upward: the
floored threshold climbs from $0.577$ to $0.922$ over training
(\Cref{thm:floor}), which is exactly why it, and only it, holds retention
bounded away from~$1$.

\begin{remark}[On Assumption~\ref{ass:scale}]
\label{rem:scale}
The scale-family model is an idealisation: real confidence distributions also
change shape (not only scale) during training. It is, however,
\emph{falsifiable} (it predicts a flat mask-ratio trajectory under the floor
and a rising one without it), and Sec.~\ref{sec:maskratio} reports that the
empirical trajectories match this prediction, which is the practical content
of the theorem. The prediction is also robust to moderate shape drift: only the
mass near the floored cutoff $s\,r_k$ affects retention, so shape changes confined
away from that quantile leave the flat-trajectory prediction intact, and the
empirical curve stays within $\pm0.03$ of its mean after epoch~$6$ despite the
distribution sharpening overall.
\end{remark}

\subsection{The Saturation Gate}
\label{sec:gate}
The calibration slice and the floor combine into a single, explicit decision. Run
the teacher once over a held-out labeled slice and measure
the \emph{reliability of its confident set},
\begin{equation}
\pi_{\mathrm{kept}}=\Pr\!\left[\hat y = y \;\middle|\; c \ge \tau\right],
\qquad \tau = 0.95,
\label{eq:pikept}
\end{equation}
the fraction of above-cutoff pixels whose pseudo-label is correct, estimated on
held-out rather than in-batch data (Proposition~\ref{prop:unbiased}). In
deployment this slice is $\mathcal{L}_{\mathrm{cal}}$, which the method already
maintains and which costs no extra labels. The gate then selects the operative
rule:
\begin{equation}
\mathrm{rule}=
\begin{cases}
\text{strict }\tau{=}0.95, & \pi_{\mathrm{kept}}\ge\tau\ \ (\text{reliable}),\\[2pt]
\text{self-adaptive floor}, & \pi_{\mathrm{kept}}<\tau\ \ (\text{unreliable}).
\end{cases}
\label{eq:gate}
\end{equation}
The criterion is a calibration test read off the teacher itself: does the retained
set match the confidence it required? When it does (a strong, well-calibrated
teacher), hard-training on the confident set is already near-optimal and relaxing
the cutoff only admits the error-enriched band (Sec.~\ref{sec:anatomy}); when it
does not (a confidently-miscalibrated teacher), hard-thresholding trusts wrong
labels at full weight and the floor's softened, confidence-weighted loss is
preferable. The decision is forward-only, needs no student update, and its
boundary is the pre-existing operating threshold $\tau$, not a value tuned on
downstream mIoU; Sec.~\ref{sec:generality} shows it picks the correct rule on six
teachers blind.
We are candid that on the standard foundation-model benchmarks
the teacher is reliable and the gate returns strict: its value is a principled,
measured \emph{decision}, not a switch that fires often.

\paragraph{How we evaluate the gate, and its one caveat}
The demonstration in Sec.~\ref{sec:generality} is a \emph{post-hoc} validation,
not a live deployment, and the distinction is worth stating plainly. We take the
six \emph{finished} strict teachers, run \texttt{scripts/gate\_stat.py} over each
dataset's validation split (forward-only, $200$ images, no gradient), and ask
whether the resulting $\pi_{\mathrm{kept}}$ predicts the sign of the
adaptive-vs-strict gap. Using the validation split rather than
$\mathcal{L}_{\mathrm{cal}}$ buys statistical precision for the validation
experiment: on Pascal 1/8 the calibration slice is only $\approx\!9$ images
(Sec.~\ref{sec:calibration}), too few to separate $98\%$ from $89\%$ with
confidence. It also means the reported $\pi_{\mathrm{kept}}$ values are
\emph{not} what a live run would have seen: they are measured on the evaluation
split and from a converged teacher. This is a validation of the criterion, not a
measurement of the deployed system, and we flag it as such
(Sec.~\ref{sec:limitations}). The criterion itself uses only quantities a live
run can compute from $\mathcal{L}_{\mathrm{cal}}$ in one forward pass, and the
separation it exploits ($98\%$ vs $89\%$) is far wider than the slice-size noise.

\subsection{Per-Class Adaptive Thresholding (Investigated Direction)}
\label{sec:perclass}

The held-out noise estimate~\eqref{eq:eps-hat} and the
floor~\eqref{eq:floor} are the ingredients of a natural, more ambitious
scheme: a fully \emph{per-class} adaptive threshold. We fit a
$\mathrm{Beta}(\alpha_k,\beta_k)$ to each class's confidence distribution by
method of moments, and choose $\tau_k$ to minimise a per-class risk that
trades estimated noise against coverage,
\begin{equation}
R_k(\tau)=\underbrace{\widehat\varepsilon_k(\tau)\,\rho_k(\tau)}_{\text{retained noise}}
+\underbrace{\lambda_k\,(1-\rho_k(\tau))}_{\text{coverage penalty}},
\label{eq:risk}
\end{equation}
where $\widehat\varepsilon_k$ is the held-out noise estimate~\eqref{eq:eps-hat}
(a point estimate: earlier drafts used a Hoeffding upper bound here, which we
dropped for the reason given in Sec.~\ref{sec:calibration} --- pixels are
spatially correlated, so a pixel-count concentration bound claims a confidence it
has not earned),
$\rho_k(\tau)=1-F_{\alpha_k,\beta_k}(\tau)$ the Beta-implied retention, and
$\lambda_k$ a coverage weight. To protect rare classes we scale the penalty by
rarity, $\lambda_k=\lambda_0\,(N_{\max}/N_k)^{\eta}$ (capped), so that
infrequent classes tolerate lower thresholds. During an initial warmup the
scheme falls back to the global dynamic threshold with the floor active; after
warmup it switches to $\arg\min_\tau R_k(\tau)$, still lower-bounded by the
per-class floor. This is the full machinery a coverage-seeking, ResNet-era
intuition recommends, built to the best of our ability, so that the audit
tests the direction and not a strawman. Sec.~\ref{sec:negative} reports that it
does not help.

\begin{algorithm}[t]
\caption{CW-BASS v2 training epoch (global-threshold variant)}
\label{alg:cwbass2}
\begin{algorithmic}[1]
\Require labeled split $\mathcal{L}_{\mathrm{tr}}$, calibration split
$\mathcal{L}_{\mathrm{cal}}$, unlabeled set $\mathcal{U}$, student $f_\theta$,
EMA teacher $f_{\theta'}$, floor scale $s$, EMA $m$
\State \textbf{Calibration pass:} reset noise histogram; for $x\in\mathcal{L}_{\mathrm{cal}}$, run $f_{\theta'}$ and accumulate $n_k,n_k^c$; set $\widehat\varepsilon_k$ via \Cref{eq:eps-hat} \Comment{unbiased (Prop.~\ref{prop:unbiased})}
\For{each unlabeled batch $x_u$ (with a labeled batch $(x_l,y_l)$)}
  \State $\mathcal{L}_x \gets \mathrm{CE}(f_\theta(x_l), y_l)$ \Comment{on $\mathcal{L}_{\mathrm{tr}}$ only}
  \State $\hat y, c \gets$ teacher posteriors on weak view of $x_u$ \Comment{\Cref{eq:pseudo}}
  \State $\bar c_t \gets m\,\bar c_t + (1{-}m)\,\mathbb{E}[c]$ \Comment{confidence EMA}
  \State $\tau^{\mathrm{dyn}} \gets$ \Cref{eq:dyn};\quad $\tau_k^{\mathrm{floor}} \gets s\,\bar c_t\,\mu_k/\max_j\mu_j$
  \State $\tau_k^{\mathrm{final}} \gets \max(\tau^{\mathrm{dyn}}, \tau_k^{\mathrm{floor}})$ \Comment{anti-collapse (Thm.~\ref{thm:floor})}
  \State $\mathcal{M} \gets \{(h,w): c_{h,w}\ge \tau_{\hat y_{h,w}}^{\mathrm{final}}\}$
  \State build strong views (aug+CutMix; feature dropout); $B\gets$ Sobel boundary of $\hat y$
  \State $\mathcal{L}_s,\mathcal{L}_{\mathrm{fp}} \gets$ CW + boundary loss on $\mathcal{M}$ \Comment{\Cref{eq:cwce}}
  \State $\mathcal{L} \gets \tfrac12(\mathcal{L}_x + \tfrac12(\mathcal{L}_s+\mathcal{L}_{\mathrm{fp}}))$; update $\theta$; EMA-update $\theta'$
\EndFor
\end{algorithmic}
\end{algorithm}

\section{Experiments and Main Results}
\label{sec:experiments}

\subsection{Datasets and Protocol}
\label{sec:setup}

\paragraph{Datasets}
Our empirical claims are made on \textbf{PASCAL~VOC 2012}~($21$ classes) under
the original (high-quality, $1464$-image) training protocol, using the same
splits as UniMatch~V2~\cite{unimatchv2} and the original CW-BASS~\cite{cwbass}
so the comparison is direct: the $1/8$ and $1/4$ partitions correspond to
$183$ and $366$ labeled images. This is deliberately the most \emph{saturated}
benchmark (UniMatch~V2 already reaches $87.9$/$88.9$~mIoU here) and is
therefore the cleanest setting in which to ask whether adaptive thresholding
adds anything on top of a strong teacher. We report mean IoU (mIoU) of the EMA
teacher, the standard evaluation model. Pascal is where we run the decisive
controls: the three-seed, strict-inclusive comparison (\Cref{tab:multiseed})
and the matched-batch trajectory dissection (\Cref{tab:negative}). To test
whether the inversion is a property of the \emph{teacher} rather than of
one benchmark, we carry the comparison to \textbf{Cityscapes} (an intermediate
teacher) and the long-tailed \textbf{ADE20K} ($150$ classes, whose teacher turns
out to be saturated but \emph{unreliable}, Sec.~\ref{sec:generality}), and across
the DINOv2-S/B/L scale family; those results are reported in
Sec.~\ref{sec:generality} (\Cref{tab:datasets}) and are what the paper's
generality claim rests on. Each dataset is run at its
standard protocol (Cityscapes crop~$686$, $120$ epochs; ADE20K crop~$518$,
$60$ epochs; both at the official batch~$16$); the ADE20K cells are
single-seed, a scope we restate where it matters (Sec.~\ref{sec:generality},
Sec.~\ref{sec:implications}).

\paragraph{Canonical protocol (``Family B'')}
To be protocol-comparable to UniMatch~V2's reported DINOv2-B numbers, our
DINOv2 runs use a two-group optimizer (no layer-wise LR decay), $60$ epochs,
weight decay $0.01$, backbone LR $5\times10^{-6}$, decoder LR
$2\times10^{-4}$, AdamW, crop $518$. The headline comparison runs all four
threshold rules at the official effective batch~$16$ (\Cref{tab:negative},
top block); additional adaptive runs at batch~$4$ probe how the failure
scales with batch (Sec.~\ref{sec:negative-confound}). For reference to the ResNet
era, the original CW-BASS protocol is ResNet-50 + DeepLabV3+, SGD (momentum
$0.9$), batch $16$, crop $321$, $80$ epochs.

\paragraph{What is controlled, and what is not}
\label{sec:whatiscontrolled}
Being exact about this matters, because the paper's central comparison is a
comparison of rules. All runs share the backbone, decoder, optimiser, LR
schedule, augmentation, EMA teacher, crop, batch and data splits. Two variables
beyond the threshold rule are \emph{not} held fixed, and we name them rather than
let the reader assume otherwise.
\begin{enumerate}[leftmargin=1.4em,itemsep=2pt,topsep=3pt]
  \item \textbf{Strict runs a different unlabeled loss.} The strict arm is the
  published UniMatch~V2 recipe (\texttt{unimatch\_v2}): plain cross-entropy on
  two \emph{independent} strong views, normalised by valid-pixel count. The five
  adaptive arms run the CW-BASS~v2 scaffold (\texttt{cwbass\_v2}): one strong
  view plus a feature-perturbation stream, confidence-weighted CE
  (\Cref{eq:cwce}) with the Sobel-boundary term, normalised by retained-pixel
  count. Strict-vs-adaptive is therefore a \emph{recipe-level} comparison of two
  published systems, not a single-factor one.
  \item \textbf{Two adaptive rules hold out labels.} The floor
  (\texttt{cwbass\_v2}) and per-class (\texttt{class\_adaptive}) rules reserve
  $\alpha{=}5\%$ of the labeled set for calibration, so they train on $174$
  Pascal 1/8 images against the other rules' $183$.
\end{enumerate}
The \emph{within-family} comparisons are single-factor: the dynamic, FreeMatch
and SoftMatch rules differ from one another in nothing but the rule, as do the
floor and per-class rules. \Cref{sec:negative-confound} bounds how much of the
strict-vs-adaptive gap the two confounds can carry.

\subsection{Implementation Details}
The DINOv2 backbone is \texttt{dinov2\_vitb14} with a DPT-lite
decoder~\cite{dpt}; training uses bf16 autocast. The two strong streams (image
+ feature perturbation) share a single backbone forward via a dual head, as in
UniMatch. Confidence weighting uses $\gamma{=}1$; the boundary term uses
weight $\beta_b{=}0.5$; calibration uses $\alpha{=}0.05$; the
floor uses $m{=}0.99$, $s{=}0.95$. The EMA teacher decays at
$\min(1-1/(t{+}1),0.996)$. Full configs are released with the code, along with
scripts that regenerate every figure and table. \Cref{tab:hyperparams} collects
the full per-dataset training configuration.

\begin{table}[t]
\centering
\caption{\textbf{Training configuration.} Settings for the reported matched-batch
runs. Backbone, decoder, optimiser, schedule, augmentation, crop, batch and
splits are shared by every arm; the strict arm additionally differs in its
unlabeled-loss form, and the floor/per-class arms in holding out $5\%$ of labels
(Sec.~\ref{sec:whatiscontrolled}).}
\label{tab:hyperparams}
\footnotesize
\setlength{\tabcolsep}{5pt}
\begin{tabular}{llll}
\toprule
Setting & Pascal~VOC & Cityscapes & ADE20K \\
\midrule
Backbone & \multicolumn{3}{c}{DINOv2-Base (ViT-B/14)} \\
Decoder & \multicolumn{3}{c}{DPT-lite} \\
Precision & \multicolumn{3}{c}{bf16 autocast} \\
Crop size & $518$ & $686$ & $518$ \\
Epochs & $60$ & $120$ & $60$ \\
Batch size & $16$ & $16$ & $16$ \\
Backbone LR & $5\!\times\!10^{-6}$ & $5\!\times\!10^{-6}$ & $5\!\times\!10^{-6}$ \\
Decoder LR & $2\!\times\!10^{-4}$ & $2\!\times\!10^{-4}$ & $2\!\times\!10^{-4}$ \\
Layer decay & $1.0$ & $1.0$ & $1.0$ \\
EMA decay & \multicolumn{3}{c}{$\min(1-1/(t{+}1),\,0.996)$} \\
\midrule
\multicolumn{4}{l}{\emph{Selection rule (the only variable):}} \\
Strict cutoff $\tau$ & \multicolumn{3}{c}{$0.95$} \\
Base threshold $\tau_0$ & \multicolumn{3}{c}{$0.6$} \\
Floor scale $s$ & \multicolumn{3}{c}{$0.95$} \\
Floor momentum $m$ & \multicolumn{3}{c}{$0.99$} \\
Calibration fraction $\alpha$ & \multicolumn{3}{c}{$0.05$} \\
Boundary weight $\beta_b$ & \multicolumn{3}{c}{$0.5$} \\
Confidence weight $\gamma$ & \multicolumn{3}{c}{$1$} \\
\bottomrule
\end{tabular}
\end{table}

\paragraph{Estimator and metric settings, in one place}
For reference: Pascal~VOC has $K{=}21$ classes; the held-out calibration uses
fraction $\alpha{=}0.05$, so the calibration slice is
$\alpha|\mathcal{L}|\!\approx\!9$ images on the 1/8 split ($\approx18$ on 1/4);
this image count, not the pixel count, is the effective sample size behind the
held-out estimator, which is why we treat it as a point estimate rather than a
bounded guarantee (Sec.~\ref{sec:calibration}). All
expected-calibration-error (ECE) figures use $40$ equal-width confidence bins
over $[0,1]$, and confidence histograms (\Cref{fig:confhist}) use $50$ bins.
We flag two caveats about the ECE comparison and do not let any conclusion rest
on it. First, ECE is bin-sensitive at these small magnitudes ($<0.03$ for both
teachers), so the $0.007$-vs-$0.024$ gap should be read as indicative, not
exact. Second, it is confounded by training state: the DINOv2 teacher is
converged whereas the ResNet-50 teacher is a representative partially-trained
model (\Cref{fig:calibration}), so the contrast is not a controlled calibration
measurement. The load-bearing statistic is instead bin-free and measured on the
converged teacher we actually use, the fraction of pixels above the cutoff
($98\%$ $\ge0.95$, \Cref{fig:confhist}). The argument depends only on that
\emph{dynamic-range collapse}, not on the teacher being better calibrated.

\paragraph{Computational overhead}
Both additions are nearly free. The floor adds only a handful of per-batch
reductions (an EMA update and per-class mean lookups), negligible against a
ViT-B forward/backward. The calibration pass is forward-only and runs once per
epoch over $\alpha|\mathcal{L}|$ images. Because $|\mathcal{L}|\!\ll\!N_U$ in
the SSSS regime (e.g.\ Pascal 1/8 has $183$ labeled images versus a $10^4$
unlabeled pool), the pass costs $\approx\!0.05\times183\approx9$ teacher
forwards per epoch, against the $649$ student forward/backward iterations
(batch~$16$ over the ${\sim}10^4$-image unlabeled pool) that make up an epoch:
under $1\%$ wall-clock overhead. No
extra parameters are introduced; the floor and calibration buffers are
$O(K)$.

\subsection{Main Results}

\paragraph{Reproduction validates the harness}
Before drawing any conclusion from the comparison we must show our implementation
is faithful, since a result from a broken baseline is worthless. Across
\emph{three seeds} the strict baseline averages $86.2{\pm}1.8$~mIoU on Pascal~VOC
1/8, its best seed reaching $87.40$, within $\sim\!0.5$ of UniMatch~V2's reported
$87.9$~\cite{unimatchv2} (\Cref{tab:multiseed}); the three-seed band brackets the
literature number, so the harness reproduces the SOTA operating point. (One seed
stalls at $84.09$, a genuine seed fragility we return to in
Sec.~\ref{sec:limitations}, not a harness fault.) The reproduction being sound, the
collapse of the adaptive variants, $3$--$5$~mIoU below strict at matched
batch~$16$, is a property of the threshold rule, not of our code.

\paragraph{What the results show}
Two subsections follow. Sec.~\ref{sec:sota} places CW-BASS~v2 in the published
SSSS landscape on all three datasets, and Sec.~\ref{sec:generality} shows what its
gate reads: on the reliable, saturated Pascal and Cityscapes teachers the gate
selects strict filtering and CW-BASS~v2 trails only UniMatch~V2, while on the one
teacher whose confident set is unreliable (ADE20K) it selects the floor and
\emph{exceeds} strict (and UniMatch~V2's reported $49.8$; single seed). For
context, the original CW-BASS reached $75.81$~mIoU on Pascal 1/8 with
ResNet-50~\cite{cwbass}; the $\sim\!12$~mIoU jump to $87.4$ is almost entirely the
DINOv2 backbone plus the strict-threshold recipe the gate selects here; the
adaptive machinery earns its keep only where the confident set is unreliable, not
on a clean saturated teacher. The controlled, batch-matched comparison that
establishes \emph{why} the gate is needed, strict dominating every adaptive rule on
a saturated teacher, the early-peak-then-collapse signature, and the mechanism
behind it, is the subject of Sec.~\ref{sec:negative}.
\subsection{CW-BASS~v2 in the SSSS Landscape}
\label{sec:sota}
\begin{table*}[!tb]
\centering
\caption{\textbf{CW-BASS~v2 in the Pascal~VOC SSSS landscape, mIoU (\%).} Classic
protocol; headers are labeled-image counts. Prior numbers from the cited papers;
the accuracy jump tracks the \emph{backbone}, not selection cleverness. On this
saturated teacher CW-BASS~v2's gate selects strict filtering, so \emph{our row is
our reproduction of the UniMatch~V2 recipe} and contains none of the adaptive
machinery --- that is the gate working as designed, and the row should be read as
a reproduction rather than as a new method beating the field. Best seed shown
(seed~$0$; $1/8$ three-seed mean $86.2{\pm}1.8$, \Cref{tab:multiseed}; $1/16$,
$1/4$ single seed; $1/2$ and Full not run for compute). Shading: \medallegend.}
\label{tab:landscape}
\small
\setlength{\tabcolsep}{5pt}
\resizebox{0.75\textwidth}{!}{%
\begin{tabular}{llccccc}
\toprule
Method & Backbone & 1/16\,(92) & 1/8\,(183) & 1/4\,(366) & 1/2\,(732) & Full\,(1464) \\
\midrule
\multicolumn{7}{l}{\emph{Specialised backbones (prior work):}} \\
Supervised baseline              & ResNet-101   & 45.1 & 55.3 & 64.8 & 69.7 & 73.5 \\
CW-BASS~\cite{cwbass}            & ResNet-50    & 72.80 & 75.81 & 76.20 & 77.15 & --   \\
ST++~\cite{stplusplus}           & ResNet-101   & 65.2 & 71.0 & 74.6 & 77.3 & 79.1 \\
U\textsuperscript{2}PL~\cite{u2pl} & ResNet-101 & 68.0 & 69.2 & 73.7 & 76.2 & 79.5 \\
PS-MT~\cite{psmt}                & ResNet-101   & 65.8 & 69.6 & 76.6 & 78.4 & 80.0 \\
AugSeg~\cite{augseg}             & ResNet-101   & 71.1 & 75.5 & 78.8 & 80.3 & 81.4 \\
UniMatch~\cite{unimatch}         & ResNet-101   & 75.2 & 77.2 & 78.8 & 79.9 & 81.2 \\
CorrMatch~\cite{corrMatch}       & ResNet-101   & 76.4 & 78.5 & 79.4 & 80.6 & 81.8 \\
DDFP~\cite{ddfp}                 & ResNet-101   & 75.0 & 78.0 & 79.5 & 81.2 & 82.0 \\
PrevMatch~\cite{prevmatch}       & ResNet-101   & 77.0 & 78.5 & 79.6 & 80.4 & 81.6 \\
BeyondPixels~\cite{beyondpixels} & ResNet-101   & 77.3 & 78.6 & 79.8 & 80.8 & 81.7 \\
AllSpark~\cite{allspark}         & MiT-B5       & 76.1 & 78.4 & 79.8 & 80.8 & 82.1 \\
SemiVL~\cite{semivl}             & CLIP ViT-B   & \rkC 84.0 & \rkC 85.6 & \rkC 86.0 & \rkB 86.7 & \rkC 87.3 \\
\midrule
\multicolumn{7}{l}{\emph{Foundation backbone (DINOv2):}} \\
UniMatch~V2~\cite{unimatchv2}    & DINOv2-S     & 79.0 & 85.5 & 85.9 & \rkB 86.7 & \rkB 87.8 \\
UniMatch~V2~\cite{unimatchv2}    & DINOv2-B     & \rkA 86.3 & \rkA 87.9 & \rkA 88.9 & \rkA 90.0 & \rkA 90.8 \\
CW-BASS~v2 (ours, gate$\to$strict) & DINOv2-B   & \rkB 84.47 & \rkB 87.40 & \rkB 88.59 & --   & --   \\
\bottomrule
\end{tabular}}
\end{table*}

For completeness we situate CW-BASS~v2 against published SSSS methods on all three
datasets: \Cref{tab:landscape} for Pascal~VOC, \Cref{tab:sota_cs} for Cityscapes,
and \Cref{tab:sota_ade} for ADE20K. Accuracy tracks the backbone throughout, and
CW-BASS~v2 sits in the top DINOv2 tier, and we are explicit about what that does
and does not mean. On the saturated Pascal and Cityscapes teachers the gate
selects strict filtering, so those rows are our reproduction of the UniMatch~V2
recipe: they carry no adaptive machinery, they trail UniMatch~V2's reported
numbers by $0.3$--$0.5$, and they are evidence that the harness is faithful, not
that the method beats the field. The row where CW-BASS~v2 is doing something of
its own is ADE20K, where the gate selects the floor and \emph{exceeds} both our
strict baseline ($+1.5$) and UniMatch~V2-B's reported $49.8$ (single seed). The
per-rule breakdown behind these rows is in \Cref{tab:multiseed,tab:datasets}.

\begin{table}[t]
\centering
\caption{\textbf{Cityscapes SSSS landscape, mIoU (\%).} Headers are labeled-image
fractions; DINOv2 rows from UniMatch~V2 and CW-BASS~v2 (ours, single seed, best
EMA, crop~$686$). On this reliable teacher the gate again selects strict, so as
on Pascal our row is a UniMatch~V2 reproduction, not an adaptive variant.
Shading: \medallegend.}
\label{tab:sota_cs}
\footnotesize
\setlength{\tabcolsep}{5pt}
\resizebox{\columnwidth}{!}{%
\begin{tabular}{llcccc}
\toprule
Method & Backbone & 1/16 & 1/8 & 1/4 & 1/2 \\
\midrule
\multicolumn{6}{l}{\emph{Specialised backbones (prior work):}} \\
CW-BASS~\cite{cwbass}            & ResNet-50  & 75.00 & 77.20 & 78.43 & --   \\
AugSeg~\cite{augseg}             & ResNet-101 & 75.2 & 77.8 & 79.6 & 80.4 \\
UniMatch~\cite{unimatch}         & ResNet-101 & 76.6 & 77.9 & 79.2 & 79.5 \\
CorrMatch~\cite{corrMatch}       & ResNet-101 & 77.3 & 78.5 & 79.4 & 80.4 \\
BeyondPixels~\cite{beyondpixels} & ResNet-101 & 78.5 & 79.2 & 80.9 & \rkC 81.3 \\
SemiVL~\cite{semivl}             & CLIP ViT-B & 77.9 & 79.4 & 80.3 & 80.6 \\
\midrule
\multicolumn{6}{l}{\emph{Foundation backbone (DINOv2):}} \\
UniMatch~V2~\cite{unimatchv2}    & DINOv2-S   & \rkC 80.6 & \rkC 81.9 & \rkC 82.4 & \rkB 82.6 \\
UniMatch~V2~\cite{unimatchv2}    & DINOv2-B   & \rkA 83.6 & \rkA 84.3 & \rkA 84.5 & \rkA 85.1 \\
CW-BASS~v2 (ours, gate$\to$strict) & DINOv2-B & \rkB 83.16 & \rkB 83.96 & \rkB 83.99 & --   \\
\bottomrule
\end{tabular}}
\end{table}

\begin{table}[t]
\centering
\caption{\textbf{ADE20K SSSS landscape, mIoU (\%).} Headers are labeled-image counts;
few methods report SSSS ADE20K. On this \emph{confidently-unreliable} teacher
($\pi_{\mathrm{kept}}{\approx}89\%$, \Cref{tab:gate}) CW-BASS~v2's gate selects the
floor, reaching $50.58$ at $1/8$ ($+1.5$ over our strict baseline and above
UniMatch~V2-B's $49.8$; single seed): the regime where adaptive filtering earns its
keep. Shading: \medallegend.}
\label{tab:sota_ade}
\footnotesize
\setlength{\tabcolsep}{4pt}
\resizebox{\columnwidth}{!}{%
\begin{tabular}{llccccc}
\toprule
Method & Backbone & 1/64\,(316) & 1/32\,(631) & 1/16\,(1263) & 1/8\,(2526) & 1/4\,(5052) \\
\midrule
UniMatch~\cite{unimatch}         & ResNet-101 & 21.6 & 28.1 & 31.5 & 34.6 & -- \\
UniMatch~\cite{unimatch}         & CLIP ViT-B & 25.3 & 31.2 & 34.4 & 38.0 & -- \\
SemiVL~\cite{semivl}             & CLIP ViT-B & \rkB 33.7 & \rkC 35.1 & \rkC 37.2 & 39.4 & -- \\
UniMatch~V2~\cite{unimatchv2}    & DINOv2-S   & \rkC 31.5 & \rkB 38.1 & \rkB 40.7 & 44.4 & \rkB 45.8 \\
UniMatch~V2~\cite{unimatchv2}    & DINOv2-B   & \rkA 38.7 & \rkA 45.0 & \rkA 46.7 & \rkB 49.8 & \rkA 52.0 \\
\midrule
Strict $\tau{=}0.95$ (ours, repro.) & DINOv2-B & --   & --   & --   & \rkC 49.10 & -- \\
CW-BASS~v2 (ours, gate$\to$floor) & DINOv2-B  & --   & --   & --   & \rkA 50.58 & -- \\
\bottomrule
\end{tabular}%
}
\end{table}

\Cref{fig:landscape} plots the Pascal~VOC landscape, making the small
adaptation penalty and the large backbone jump visible at a glance.

\begin{figure}[!htb]
\centering
\includegraphics[width=0.9\columnwidth]{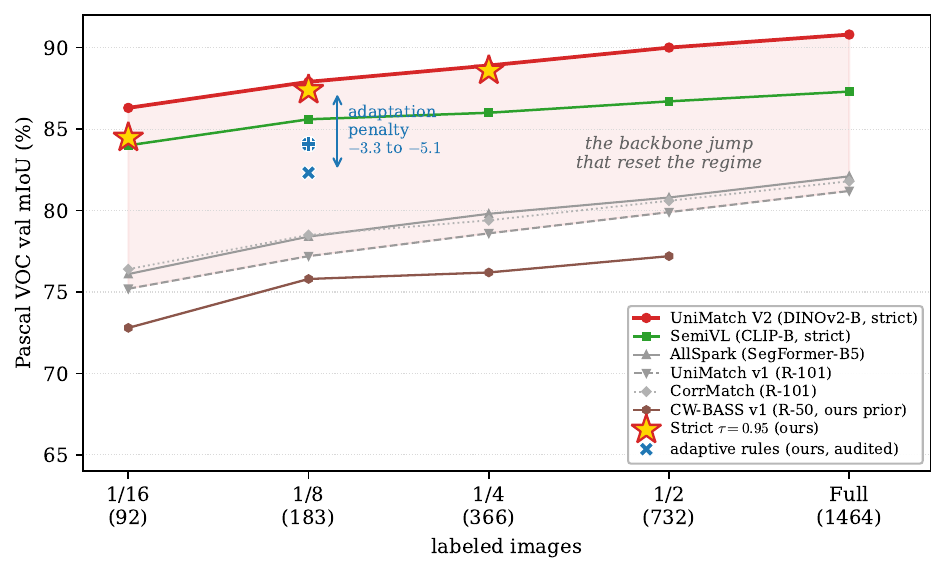}
\caption{\textbf{Pascal~VOC SSSS landscape} (val mIoU vs.\ labeled-image budget).
CW-BASS~v2's strict operating point ($\star$) sits in the top DINOv2 tier; the
adaptive rules (audited, $\times$) fall a fixed $3.3$--$5.1$~mIoU below it (shaded
\emph{adaptation penalty}), a gap dwarfed by the backbone jump over the
ResNet-101 field that reset the regime.}
\label{fig:landscape}
\end{figure}

\subsection{Generality and the Operationalized Gate}
\label{sec:generality}
A single-backbone, single-dataset result cannot support a regime-level claim, so
we sweep the DINOv2 scale family (ViT-S/B/L) and three datasets.
\Cref{tab:datasets} gives the mIoU: at matched batch~$16$, strict beats every
adaptive rule on Pascal (by $3$--$5$, at S/B/L alike), the rules \emph{tie} strict
on Cityscapes (spread ${\le}0.7$), and the floor \emph{edges ahead} on ADE20K
(DINOv2-B $50.58$ vs $49.10$; single seed). What this section establishes is what
\emph{predicts} that sign-flip, and hence what the gate should read.

\paragraph{What the gate reads: confident-set reliability, not raw saturation}
It is tempting to attribute the ADE result to an \emph{under-confident} teacher
that preserves the confidence spread adaptive rules were built to exploit. Direct
measurement refutes that. \Cref{tab:gate} reports each strict teacher's confidence
geometry, measured forward-only on held-out labels (Sec.~\ref{sec:gate}): all six teachers are \emph{highly
saturated} (${\ge}82\%$ of pixels at $c{\ge}0.95$; ADE is only ${\sim}13$~points
below Pascal), so under-confidence cannot explain the ADE win. What separates the
regimes is a sharper statistic, $\pi_{\mathrm{kept}}\!=\!\Pr[\text{correct}\mid
c{\ge}0.95]$, the \emph{reliability of the set strict keeps}. On Pascal and
Cityscapes $\pi_{\mathrm{kept}}{\approx}98\%$: the confident predictions are almost
all right, so hard-training on them is optimal and lowering the cutoff only admits
the error-enriched band (Sec.~\ref{sec:anatomy}). On ADE
$\pi_{\mathrm{kept}}{\approx}89\%$: the teacher is \emph{confidently wrong} on
roughly one confident pixel in nine, so hard-thresholding trusts those errors at
full weight and the floor's confidence-weighted, softened loss mitigates them.
This is exactly the gate criterion of Sec.~\ref{sec:gate}: \emph{use strict when
the retained set is at least as accurate as the confidence demanded
($\pi_{\mathrm{kept}}{\ge}\tau$), and engage the adaptive floor otherwise.} The
boundary is the operating threshold $\tau{=}0.95$ itself, a calibration criterion,
not a value tuned to mIoU, and \Cref{fig:gate} shows it makes the correct
strict-vs-floor call on all six teachers \emph{blind}, from confidence and
held-out labels alone.

\paragraph{Scope of this claim} $\pi_{\mathrm{kept}}$ is a property of the
teacher's calibration, not of DINOv2 in particular, so the gate is in principle
backbone-agnostic; but every teacher we can test is DINOv2-family, so we claim the
demonstration only for \emph{DINOv2-family teachers} (Sec.~\ref{sec:limitations}),
and with six teachers we can show the criterion \emph{separates} the regimes, not
pin its exact boundary. The ADE win is single-seed: the \emph{sign} (no collapse;
floor ${\ge}$ strict) reproduces across S and B, but its magnitude is within
plausible seed noise, and we do not rest the paper on it. The floor, CW-BASS~v2's
own mechanism, is \emph{worst} on the reliable Pascal teacher yet \emph{best} on
the unreliable ADE teacher: regime-appropriate, which is what the gate exploits.
(DINOv2-L is a partial cross-check: it confirms the Pascal loss in full and its
L-ADE dynamic run reaches $52.00$ with no collapse, but strict L-ADE exceeded GPU
memory and L-Cityscapes was not run, so we claim no signed L gap off Pascal.)

\begin{table}[t]
\centering
\caption{\textbf{Generality across backbone scale and dataset.} Best EMA mIoU at
matched batch~$16$, seed~$0$; Enc.\ = DINOv2-\{S,B,L\} (DINOv2-B Pascal~1/8
three-seed mean $86.19{\pm}1.82$, \Cref{tab:multiseed}). Strict beats every
adaptive rule on Pascal at all scales, the rules tie on Cityscapes, and the floor
edges ahead on the \emph{confidently-unreliable} ADE20K teacher (the separator is
$\pi_{\mathrm{kept}}$, \Cref{tab:gate}; ADE single seed). Per-class was run only at
DINOv2-B. \textbf{Best} per row is bold, \emph{second} is italic.}
\label{tab:datasets}
\footnotesize
\setlength{\tabcolsep}{4pt}
\begin{tabular}{llcccc}
\toprule
Enc. & Dataset (1/8, b16) & Strict & Dynamic & Floor & Per-cls \\
\midrule
\multirow{3}{*}{S}
 & Pascal~VOC  & $\mathbf{85.00}$ & $\mathit{82.00}$ & $80.48$ & --- \\
 & Cityscapes  & $\mathit{81.51}$ & $\mathbf{81.86}$ & $81.16$ & --- \\
 & ADE20K      & $44.02$ & $\mathbf{44.92}$ & $\mathit{44.69}$ & --- \\
\midrule
\multirow{3}{*}{B}
 & Pascal~VOC  & $\mathbf{87.40}$ & $\mathit{84.07}$ & $82.32$ & $\mathit{84.07}$ \\
 & Cityscapes  & $\mathbf{83.96}$ & $83.43$ & $\mathit{83.47}$ & $83.40$ \\
 & ADE20K      & $49.10$ & $\mathit{49.67}$ & $\mathbf{50.58}$ & $49.25$ \\
\midrule
\multirow{2}{*}{L}
 & Pascal~VOC        & $\mathbf{86.91}$ & $\mathit{84.88}$ & $83.50$ & --- \\
 & ADE20K$^{\dagger}$ & --- & $52.00$ & --- & --- \\
\bottomrule
\end{tabular}
\\[2pt]
{\footnotesize $^{\dagger}$Partial: strict L-ADE OOM'd and L-Cityscapes was not
run; the L-ADE dynamic rule at $52.00$ shows no collapse, matching the
unreliable-teacher behaviour at S and B, but with no strict arm we claim no
signed L-ADE gap.}
\end{table}

\begin{table}[t]
\centering
\caption{\textbf{The gate, measured blind on six strict teachers.} Confidence
geometry of each converged strict teacher, measured forward-only on that
dataset's validation split ($200$ images; see Sec.~\ref{sec:gate} for why the
validation split rather than $\mathcal{L}_{\mathrm{cal}}$, and what that does and
does not show). Values in \%: saturation $S{=}\Pr[c{\ge}0.95]$
and reliability $\pi_{\mathrm{kept}}{=}\Pr[\text{correct}\mid c{\ge}0.95]$. All
teachers are saturated; what tracks the sign of the adaptive-vs-strict gap
($\Delta{=}$\,floor$-$strict mIoU) is $\pi_{\mathrm{kept}}$, \emph{not} $S$. The
gate ($\pi_{\mathrm{kept}}{\ge}\tau{\Rightarrow}$\,strict, else adaptive;
$\tau{=}0.95$) makes the correct call on every teacher without seeing $\Delta$.
The gate is binary: it does not rank the dynamic and per-class rules, which on
Cityscapes-S and ADE-S edge past the floor (\Cref{tab:datasets}).}
\label{tab:gate}
\footnotesize
\setlength{\tabcolsep}{4.5pt}
\begin{tabular}{llcccc}
\toprule
Teacher & Enc. & $S{\ge}0.95$ & $\pi_{\mathrm{kept}}$ & gate & $\Delta$ \\
\midrule
Pascal~VOC & B & $97.8$ & $98.5$ & strict & $-5.08$ \\
Pascal~VOC & S & $96.7$ & $98.4$ & strict & $-4.52$ \\
Cityscapes & B & $92.4$ & $98.4$ & strict & $-0.49$ \\
Cityscapes & S & $92.0$ & $98.1$ & strict & $-0.35$ \\
ADE20K     & B & $85.0$ & $\mathbf{89.3}$ & \textbf{adaptive} & $\mathbf{+1.48}$ \\
ADE20K     & S & $82.1$ & $\mathbf{89.2}$ & \textbf{adaptive} & $\mathbf{+0.67}$ \\
\bottomrule
\end{tabular}
\end{table}

\begin{figure}[!htb]
\centering
\includegraphics[width=0.85\linewidth]{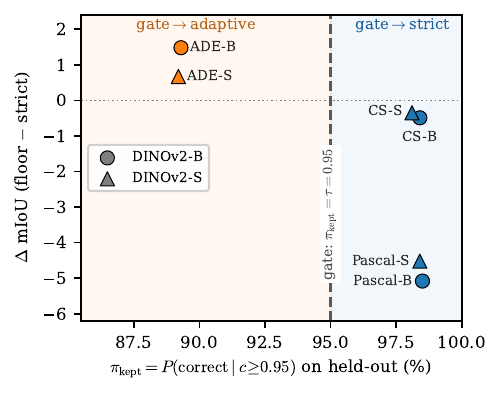}
\caption{\textbf{The reliability gate, demonstrated blind.} Adaptive-vs-strict gap
($\Delta{=}$\,floor$-$strict mIoU) against $\pi_{\mathrm{kept}}$, measured one-pass
on held-out labels with no access to $\Delta$. All six strict teachers fall on the
correct side of the boundary $\pi_{\mathrm{kept}}{=}\tau{=}0.95$: where the
confident set is reliable (right; Pascal/Cityscapes) adaptive filtering does not
help; where it is unreliable (left; ADE20K) the floor is competitive. ADE single
seed.}
\label{fig:gate}
\end{figure}

\subsection{Qualitative Results}
\label{sec:qualitative}

\Cref{fig:qualitative} contrasts strict and CW-BASS~v2's adaptive rule on two
\emph{favourable} validation images per dataset, chosen by per-image mIoU
advantage for the adaptive rule (the per-class rule on Pascal, the self-adaptive
floor on Cityscapes/ADE20K). These are deliberately selected cases; what they
illustrate is that the visual contrast tracks confident-set \emph{reliability},
not any single image. The Pascal rows make the point against us and we leave them
that way: even chosen for the adaptive rule's benefit, its maps carry visible
speckle that strict does not, and strict leads \emph{in aggregate} besides. The
audit rests on the aggregate (\Cref{tab:negative,tab:tail}). On Cityscapes the maps are
near-identical, the visual signature of the aggregate near-tie. On the
confidently-unreliable ADE20K teacher (${\sim}89\%$ reliable, \Cref{tab:gate}) the
floor recovers large regions strict mislabels, and, unlike Pascal, these
per-image wins run in the \emph{same} direction as the aggregate (floor $50.58$
vs.\ strict $49.10$, single seed; Sec.~\ref{sec:generality}).

\begin{figure*}[!tb]
\centering
\IfFileExists{figures/qualitative_combined.pdf}{%
  \includegraphics[width=0.55\textwidth]{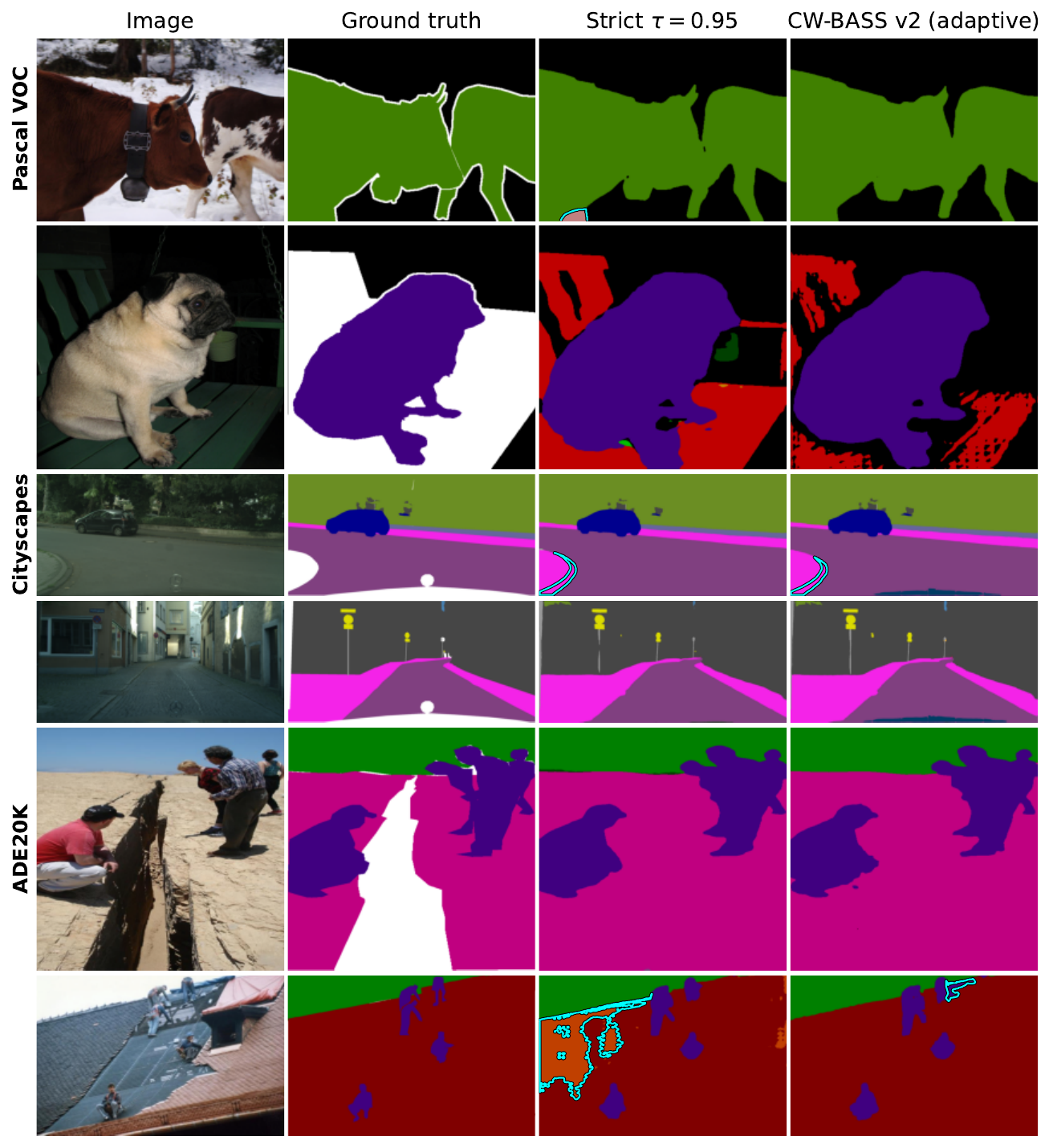}%
}{%
  \fbox{\parbox{0.95\linewidth}{\centering\color{gray}\vspace{1.6cm}
  \textbf{Figure placeholder} — combined qualitative grid. Generate with
  \texttt{PYTHONPATH=. python scripts/plot\_qualitative\_combined.py}.\vspace{1.6cm}}}%
}
\caption{\textbf{Qualitative comparison across the three teachers.} (DINOv2-Base,
EMA teachers; two \emph{favourable} validation rows per dataset, selected by
per-image mIoU advantage for the adaptive rule.) Columns: input, ground truth,
strict $\tau{=}0.95$, and CW-BASS~v2's adaptive rule (per-class on Pascal, the
self-adaptive floor on Cityscapes/ADE20K); \textcolor{cyan!70!black}{cyan}
contours mark large prediction-vs-GT error regions (cyan rather than red: the
Pascal and ADE20K palettes both contain a saturated red). Row heights are
equalised, so panels are mildly stretched. Read these as illustration, not
evidence: the aggregate verdicts are in \Cref{tab:negative,tab:tail,tab:datasets}.
Even on rows chosen to favour it, the adaptive rule is visibly noisier than
strict on the reliable Pascal teacher (row~2, scattered speckle); the two are
near-indistinguishable at the Cityscapes tie; and on the confidently-unreliable
ADE20K teacher the floor removes a large mislabeled region strict introduces
(row~6), the one dataset where its per-image wins run with the aggregate
(single seed).}
\label{fig:qualitative}
\end{figure*}

\subsection{Ablations}
\label{sec:ablation}

\Cref{tab:ablation} adds the CW-BASS mechanisms cumulatively on Pascal~VOC 1/8
(DINOv2-Base) in a reduced-scale single-seed study, which isolates the
\emph{within-adaptive-family} contributions that the full-scale
\Cref{tab:negative} cannot (since there the whole family collapses). Two audit
findings stand out. First, the boundary auxiliary, which clearly helped the
original ResNet-era CW-BASS~\cite{cwbass}, does \emph{not} transfer to DINOv2
($-1.9$~mIoU here): its benefit is backbone-dependent. Confidence weighting
gives a modest $+1.2$, and held-out calibration is accuracy-neutral ($-0.5$;
its value is the unbiased noise estimate, not mIoU; Sec.~\ref{sec:negative}).
Second, within the dynamic-threshold family the floor is the largest single
effect ($+2.7$), consistent with its stability role (Sec.~\ref{sec:analysis}).
The crucial caveat (and the reason this ablation is a supporting study, not
the headline) is that this entire ladder lives \emph{below} the strict fixed
threshold: the best configuration here ($78.8$, reduced scale) and the
floor-stabilised matched-batch~$16$ run ($82.32$, \Cref{tab:negative}) fall
$\sim\!8$ and $\sim\!5$~mIoU short of the strict baseline's $87.4$
respectively. Optimising within the
adaptive family does not change the conclusion that the family is the wrong
choice.

\begin{table}[t]
\centering
\caption{\textbf{Cumulative ablation on Pascal~VOC 1/8.} (DINOv2-Base,
reduced-scale, single seed.) CWLoss = confidence-weighted CE; Bnd = boundary auxiliary;
Floor = self-adaptive floor; Calib = held-out calibration. The boundary
auxiliary, which helped at ResNet strength, does \emph{not} transfer ($-1.9$); the
floor is the largest within-family effect ($+2.7$); but the entire ladder sits
below strict (\Cref{tab:negative}), and optimising within the adaptive family does
not change which family wins.}
\label{tab:ablation}
\footnotesize
\begin{tabular}{ccccc}
\toprule
CWLoss & Bnd & Floor & Calib & mIoU \\
\midrule
           &            &            &            & 76.75 \\
\checkmark &            &            &            & 77.93 \\
\checkmark & \checkmark &            &            & 76.08 \\
\checkmark & \checkmark & \checkmark &            & \textbf{78.76} \\
\checkmark & \checkmark & \checkmark & \checkmark & 78.28 \\
\bottomrule
\end{tabular}
\end{table}

\section{Analysis: Why the Gate}
\label{sec:negative}

This analysis is the evidence CW-BASS~v2's gate is built on: it establishes
\emph{when} adaptive filtering fails and the mechanism that makes the failure
predictable, so the method can select against it. That a strict high threshold
beats coverage-seeking filtering on a DINOv2 teacher is, on its own, not new:
UniMatch~V2~\cite{unimatchv2} already established it, and we \emph{confirm} rather
than discover it. Our contribution here is not the ranking but its
\emph{explanation}: a batch-matched decomposition of \emph{why} it holds,
measured link by link, from which the gate's reliability criterion follows.
On
the identical DINOv2-Base backbone we compare four settings: (i)~the strict fixed
threshold $\tau{=}0.95$ of UniMatch~V2~\cite{unimatchv2}; (ii)~the original
CW-BASS dynamic global threshold (\Cref{eq:dyn}); (iii)~that dynamic threshold
lower-bounded by the self-adaptive floor and informed by held-out calibration
(the full \texttt{cwbass\_v2} recipe); and (iv)~the per-class adaptive scheme
of~Sec.~\ref{sec:perclass} (\texttt{class\_adaptive}). Backbone, decoder,
optimiser, augmentation, EMA teacher, crop, batch and splits are held fixed;
rules (ii)--(iv) differ from one another in the rule alone, while (i) is the
UniMatch~V2 recipe and also differs in its unlabeled-loss form
(Sec.~\ref{sec:whatiscontrolled}).

\paragraph{Finding}
On a \emph{full-scale, three-seed} comparison at the official batch~$16$, strict
included (\Cref{tab:multiseed}), the strict fixed threshold leads every adaptive
rule \emph{on average} and is the only rule that ever reaches the UniMatch~V2
operating point. Averaged over three seeds it reaches $86.19\pm1.82$~mIoU on
Pascal~VOC 1/8 and attains the ${\sim}87.4$ mode on two of three seeds, while all
\emph{five} adaptive rules sit below it in a $79$--$85$ band. The comparison
turns on strict's \emph{upside}, not on a clean sweep of the seed grid, and we
are precise about which is which. Across the fifteen adaptive runs not one
reaches strict's mode: the best adaptive result on any seed is $84.52$, a hard
ceiling $2.9$ below strict's $87.40$. But strict's own seed variance is real and
it costs it one head-to-head: on seed~$1$ strict stalls at $84.09$ and finishes
\emph{below} the dynamic rule's $84.46$ on that seed. That is the single cell in
the $3\times6$ grid where an adaptive rule beats strict, and it is a comparison
between a stalled strict run and the strongest adaptive rule; we report it rather
than claim a clean sweep. The margin over the closest rule (dynamic, $+1.84$) is
within strict's seed spread and, by a two-sample Welch $t$-test at three seeds,
not significant ($p{=}0.22$), as are the margins over per-class ($p{=}0.14$) and
SoftMatch ($p{=}0.08$); only the margins over the floor ($+4.3$, $p{=}0.044$) and
FreeMatch ($+6.8$, $p{=}0.017$) clear $0.05$. FreeMatch and SoftMatch, run by
\emph{their own} definitions, fail the same way as our reconstructions, so the
result is not an artefact of the CW-BASS formula.

At matched batch~$16$ each adaptive rule reaches its best EMA checkpoint early (FreeMatch at
epoch~$1$, the floor and per-class rules within the first few epochs) and then drifts down,
versus strict's late peak (epochs $32$--$42$ on its climbing seeds; \Cref{tab:negative},
\Cref{fig:trajectory}). On the single seed dissected in \Cref{tab:negative}, the per-class
rule collapses hardest, and the collapse is not confined to the student: the student sheds
$6.0$~mIoU and the EMA teacher $6.14$, so the slow EMA does \emph{not} damp the fall here,
it merely delays it. A note on strict's variance: its three
seeds are $87.40/84.09/87.08$, so two climb to the SOTA mode and one stalls early at $84.09$;
the stall is genuine (the seeds diverge by epoch~$8$, \Cref{fig:multiseed}), and we report it
rather than hide it. It means strict's \emph{margin} is seed-dependent, and on that one seed
its \emph{ranking} is too.

\paragraph{Unequal run lengths}
Runs were trained with early stopping (patience $20$ epochs on EMA mIoU) inside a
$60$-epoch budget, so they terminate at different epochs: strict seed~$0$ ran the
full $60$, the dynamic rule stopped at $51$, the floor and per-class rules at
$24$, and strict seeds~$1$/$2$ at $28$/$52$. This does not affect any best-EMA
number (every run had run $20$ epochs past its peak by construction) but it does
mean the ``final'' column of \Cref{tab:negative} is each run's last trained epoch
rather than a common epoch, and it is why the trajectories in
\Cref{fig:trajectory,fig:multiseed} end at different points. The early stopping
is part of the training protocol, not a post-hoc analysis choice.

\begin{figure}[!htb]
\centering
\includegraphics[width=0.85\linewidth]{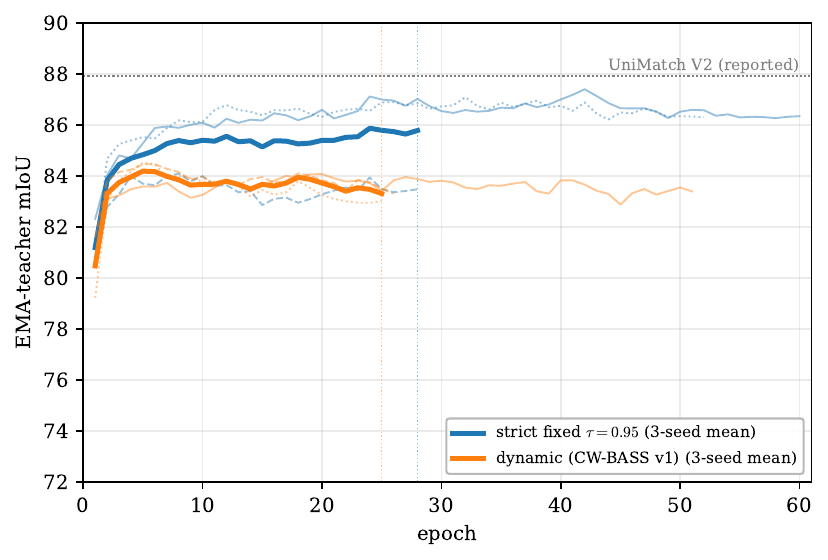}
\caption{\textbf{The decisive control, visualized.} EMA-teacher mIoU over training
(Pascal~VOC 1/8, DINOv2-Base; thin lines per seed, bold the 3-seed mean). Strict
(blue) separates from the dynamic band (orange) by epoch~$8$ on \emph{two} of its
three seeds and holds that gap for the rest of training. The third is the
exception the mean has to carry: strict's stalled seed~$1$ (dashed blue) drops
into the orange band early and stays inside it, finishing below the best dynamic
seed --- visible here as the one blue line that never leaves the cluster. The two
bold means are computed only over the \emph{all-live} window, the epochs every one
of the three seeds logged (dotted vertical lines), matching the seeds' different
stopping points (\Cref{tab:multiseed}); thin per-seed lines continue past that
window where a seed ran longer.}
\label{fig:multiseed}
\end{figure}

\begin{table}[t]
\centering
\caption{\textbf{The decisive control: three-seed comparison at matched
batch~$16$} (Pascal~VOC 1/8, DINOv2-Base; best EMA mIoU per seed, mean$\pm$std;
(v1)/(v2) mark the CW-BASS generations). Every adaptive rule sits below strict
\emph{on average}, and none reaches strict's ${\sim}87.4$ mode (best adaptive:
$84.52$). The one head-to-head strict loses is seed~$1$, where its stalled run
($84.09$) falls below the dynamic rule's $84.46$ (underlined). Welch $t$-tests at
three seeds: the margins over the floor ($p{=}0.044$) and FreeMatch ($p{=}0.017$)
are significant; those over dynamic, per-class and SoftMatch are not
($p{=}0.22/0.14/0.08$). \textbf{Best} is bold, \emph{second} is italic (mean).}
\label{tab:multiseed}
\footnotesize
\setlength{\tabcolsep}{3pt}
\begin{tabular}{lcccc}
\toprule
Threshold rule & seed 0 & seed 1 & seed 2 & mean$\pm$std \\
\midrule
Strict fixed $\tau{=}0.95$ \cite{unimatchv2} & $87.40$ & $84.09$ & $87.08$ & $\mathbf{86.19\pm1.82}$ \\
Dynamic (v1)                    & $84.07$ & $\underline{84.46}$ & $84.52$ & $\mathit{84.35}\pm0.24$ \\
Per-class adaptive (v2)         & $84.07$ & $83.28$ & $84.02$ & $83.79\pm0.43$ \\
SoftMatch~\cite{softmatch}      & $83.26$ & $83.47$ & $82.24$ & $82.99\pm0.66$ \\
Dynamic\,+\,floor\,+\,calib (v2) & $82.32$ & $81.23$ & $82.10$ & $81.88\pm0.57$ \\
FreeMatch~\cite{freematch}      & $79.10$ & $79.95$ & $79.02$ & $79.36\pm0.51$ \\
\bottomrule
\end{tabular}
\end{table}

The same early-peak signature recurs, more violently, at small batch: at batch~$4$
the three adaptive rules peak within $1$--$5$ epochs \emph{below $80.5$} and the
raw student then falls steeply (\Cref{tab:negative}, lower blocks). We read these
runs \emph{qualitatively} only: they mix batch sizes with the strict baseline and
were cut at a compute deadline (Sec.~\ref{sec:negative-confound}), so we claim from
them just the \emph{direction} (adaptive loses, peaks early, degrades harder when
each noisy update carries more weight), not a numeric severity gradient. The
early-peak-then-decline shape is the fingerprint of confirmation
bias~\cite{arazo2020pseudo} under an over-permissive pseudo-label stream
(\Cref{cor:collapse}); the quantitative claim rests on the matched batch~$16$
comparison above.

\begin{table}[t]
\centering
\caption{\textbf{At matched batch, every adaptive threshold loses to strict, and
the floor loses worst.} Pascal~VOC mIoU (DINOv2-Base). Rows 2--4 differ from one
another in the threshold rule alone; row~1 is the UniMatch~V2 recipe
(Sec.~\ref{sec:whatiscontrolled}).
``best EMA (ep)'' is the best EMA mIoU and its epoch, ``EMA/raw drop'' the
best-to-final decline of the EMA teacher and of the student. \emph{Top block}:
matched batch~$16$ (the controlled comparison), strict climbs to $87.40$
(ep~$42$) while adaptive rules peak early (ep~$4$--$20$) then decline, the
per-class model shedding $6.14$~mIoU. Single-seed trajectories; three-seed bests are in \Cref{tab:multiseed}, and the
batch-$4$ runs, which we read qualitatively only, are in the supplementary
material (\Cref{tab:negative_b4}).}
\label{tab:negative}
\footnotesize
\setlength{\tabcolsep}{2.4pt}
\begin{tabular}{llccc}
\toprule
Threshold rule & batch & best EMA (ep) & EMA drop & raw drop \\
\midrule
\multicolumn{5}{l}{\emph{Pascal~VOC 1/8 (183 labels), matched batch~16:}}\\
Strict fixed $\tau{=}0.95$ \cite{unimatchv2} & 16 & $\mathbf{87.40}$ (42) & $1.06$ & $1.8$ \\
Dynamic (CW-BASS)        & 16 & $84.07$ (20) & $0.68$ & $1.4$ \\
Dynamic + floor + calib  & 16 & $82.32$ (4)  & $2.24$ & $2.4$ \\
Per-class adaptive       & 16 & $84.07$ (4)  & $\mathbf{6.14}$ & $6.0$ \\

\bottomrule
\end{tabular}
\end{table}

A note on the degradation columns. Because the EMA teacher moves slowly it can lag a
collapsing student, so we report the fall of each. The \emph{EMA drop} is the evaluation
model degrading from its own best to its final epoch: benign for strict ($1.06$) but a
striking $6.14$ for the per-class rule ($84.07{\to}77.93$). That $6.14$ is the sharpest
number in the table: the EMA teacher not only fails to damp the student's $6.0$ fall, it
exceeds it, so it is the actual evaluation model that collapses --- the direct signature of
``early peak then decline''. The \emph{raw drop}
reports the underlying student's peak-to-final fall, mild at batch~$16$ but severe at
batch~$4$ (\Cref{tab:negative_b4}, supplementary material).

The dynamic and per-class rules report the same best EMA to two decimals ($84.07$)
by coincidence, not duplication: the underlying values are $84.071$ (dynamic, epoch~$20$)
and $84.068$ (per-class, epoch~$4$), two distinct runs that peak at nearly the same height
by different routes, the per-class rule far earlier, then collapsing far harder.

\subsection{What Else Could Explain the Gap}
\label{sec:negative-confound}
Three alternative explanations deserve to be ruled out or bounded: batch size,
the strict arm's different unlabeled loss, and the calibration split.

\paragraph{Not the batch}
The primary, fully on-disk control is the top block of
\Cref{tab:negative}: all four runs share the \emph{same} batch~$16$, and strict
beats dynamic by $3.3$ and the floor by $5.1$~mIoU, with the same
early-peak-then-decline versus monotone-climb trajectory difference, so the gap
is not a batch artefact. Batch size appears to control the
\emph{severity} of the collapse, the adaptive rules fall harder at batch~$4$
(\Cref{tab:negative}, lower blocks), but we report that only \emph{qualitatively}:
the batch-$4$ runs mix batch sizes relative to the strict baseline and were cut at
a compute deadline, so we do not attach a numeric severity gradient to them. The
claim we do rest on is batch-invariant and fully batch-matched: at the official
batch~$16$, confidence-adaptive thresholds lose to a strict cutoff at
foundation-model strength, and the floor does not rescue them.

\paragraph{Not (mostly) the loss form or the calibration split}
Strict is the UniMatch~V2 recipe and so also differs from the adaptive arms in
its unlabeled loss, and two adaptive rules hold out $5\%$ of labels
(Sec.~\ref{sec:whatiscontrolled}). Both confounds are measured, and neither is
close to large enough. In the cumulative ablation (\Cref{tab:ablation}) the two
CW-BASS loss terms together move accuracy by $-0.7$ ($76.75$ without them,
$76.08$ with both), and the calibration split by $-0.5$ (\Cref{tab:alpha}); at
face value they account for ${\approx}1.2$ of the $3.3$--$5.1$~mIoU gap, leaving
at least $2$~mIoU that they cannot explain. The stronger argument is that the
rule alone demonstrably moves accuracy by \emph{more} than the whole
strict-vs-adaptive gap: among the five adaptive rules, which share one loop, one
loss and one config and differ in nothing but the rule, three-seed means span
$79.36$ (FreeMatch) to $84.35$ (dynamic), a $5.0$~mIoU spread
(\Cref{tab:multiseed}). A selection rule is thus an order-of-magnitude-larger
lever than the recipe difference that separates strict from the family. What we
cannot bound from our runs is the dual-strong-view difference, which we do not
ablate; a drop-in strict arm inside the CW-BASS~v2 loop would settle it and is
the first item of future work (Sec.~\ref{sec:limitations}).

\paragraph{Is it the stale constants?}
A natural objection is that the dynamic rule loses only because it is run at its
original ResNet-era constants ($\tau_0{=}0.6$, $\beta{=}0.5$, $\tau_{\min}{=}0.3$),
which by construction ceiling its cutoff near $0.34$ (\Cref{cor:collapse}), and a
re-tuned rule might close the gap. A parameter sweep on Pascal~VOC 1/8 at
batch~$16$ (supplementary material, \Cref{tab:taus}) rules this out: raising the
lower clamp $\tau_{\min}$ through $0.5$, $0.7$, $0.9$, or the base $\tau_0$ to
$1.2$/$1.7$, leaves the rule in the same $81$--$85$ band, still peaking early
(epochs $2$--$9$), and never within $2.7$~mIoU of strict. The one setting that
would close the gap is the degenerate one: at $\tau_{\min}{=}0.95$ the clamp
pins the cutoff at the strict threshold and the rule \emph{is} the strict rule,
with no adaptive range left, which is an algebraic identity rather than a cell we
ran. The deficit is therefore \emph{not} an artefact of a strawman
parameterisation: across the whole sweep no configuration retaining any adaptive
range comes close to strict. This is
the central point sharpened: once the confidence range has collapsed
(Sec.~\ref{sec:anatomy}), there is no setting of an adapt-downward rule that beats
a fixed high cutoff, because any room left to adapt downward is room to admit the
error-enriched band beneath the saturated mass.

\subsection{Anatomy of the Collapse}
\label{sec:anatomy}

The failure is not a tuning artefact; it is a causal chain that any
confidence-adaptive rule triggers under a foundation-model teacher
(schematised in \Cref{fig:teaser}).
We walk it link by link. Every quantity below
is read off the TensorBoard logs of the matched-batch runs, not asserted, and the
chain compresses to one measured line:

\keybox{\textbf{The measured chain.}\enspace
\emph{Confidence saturation} ($98\%$ of pixels $\ge0.95$, ECE $0.007$)
$\Rightarrow$ \emph{dynamic-range collapse} (the dynamic cutoff pinned in
$[0.300,0.331]$, directly beneath the analytic ceiling
$\tau_0\sigma(\beta/2)\approx0.34$ of \Cref{cor:collapse})
$\Rightarrow$ \emph{mask flooding} (retention crosses $0.95$ by epoch~$7$ and
$0.99$ by epoch~$8$, reaching $1.000$; $99.9\%$ of teacher errors admitted
vs.\ $63\%$ for strict)
$\Rightarrow$ \emph{early peak} (best EMA at epochs~$4$--$20$ for the
adaptive rules vs.\ epoch~$42$ for strict)
$\Rightarrow$ \emph{decline} (the per-class evaluation model itself sheds
$6.14$~mIoU, $84.07{\to}77.93$).\enspace
Strict $\tau{=}0.95$ breaks the chain at the second link and never enters it.}

\paragraph{Link 1: confidence saturation}
\Cref{fig:confhist} histograms the strict teacher's per-pixel max-softmax
confidence over the validation set. $98\%$ of valid pixels sit at confidence
$\ge0.95$: the distribution is pinned against~$1$. This is the regime
inversion. Adaptive thresholding (FlexMatch~\cite{flexmatch},
FreeMatch~\cite{freematch}) was built for the opposite case: a teacher with
many \emph{under}-confident-but-correct predictions that a high cutoff would
waste. Here there are almost none to recover.

\paragraph{Link 2: dynamic-range collapse, where the threshold stops
discriminating}
With the confidence mass piled at~$1$, a cutoff's job is decided entirely by
where it sits relative to that pile. \Cref{fig:mechanism}a is the direct
measurement: across the whole run the bare dynamic threshold stays pinned in
$[0.300,0.331]$ (right at the analytic ceiling $\tau_0\sigma(\beta/2)\approx0.34$
of \Cref{cor:collapse}) while the teacher's mean confidence climbs past
$0.88$. The gap is enormous, so essentially every pixel clears the
cutoff: \Cref{cor:collapse} made empirical. The per-class rule behaves the
same way at first (rising only to $0.805$); only the floor (\Cref{fig:mechanism}b)
forces the threshold to track confidence upward (\Cref{thm:floor}), climbing
$0.577\to0.922$, and even then it bounds retention at $\approx\!0.91$, still
admitting most of the (saturated) mass.

\paragraph{Link 3: mask flooding, and the admitted pixels are
disproportionately wrong}
A non-discriminative threshold retains nearly everything. The matched-batch~$16$
dynamic run starts at retention $\rho_t{=}0.569$, crosses $0.90$ and $0.95$ by
epoch~$7$, and reaches $0.99$ by epoch~$8$ ($\rho_t{\to}1.000$ thereafter);
the per-class rule crosses $0.90$ by epoch~$8$ (\Cref{tab:retention},
\Cref{fig:trajectory}b). The band the dynamic
rule admits but a strict cutoff rejects (confidence in $[0.33,0.95]$) is where
the teacher's \emph{errors} concentrate (the red mass in \Cref{fig:confhist}):
the dynamic threshold admits $99.9\%$ of all teacher errors versus $63\%$ for
$\tau{=}0.95$. Where the teacher is uncertain it is uncertain because it is
\emph{wrong}, not merely conservative, so the extra coverage is almost pure
noise. This is exactly the prediction of Proposition~\ref{prop:unbiased} (that
the per-class noise rate $\widehat\varepsilon_k$ should \emph{not} fall for the
classes whose thresholds the adaptive rule lowers), and our per-class IoU
evidence is consistent with it: the hard classes the rarity-scaled rule targets
do not improve (\Cref{tab:tail}; the per-class tail mean $58.6$ does not exceed
the bare dynamic threshold's $59.2$, a $0.6$~mIoU difference well inside the
seed spread of Sec.~\ref{sec:analysis}, i.e.\ a tie, so the targeted coverage
gain simply does not materialise). We lean on the measured per-class IoU rather
than a direct read of $\widehat\varepsilon_k$, whose trajectory we leave to
future work (Sec.~\ref{sec:analysis}).

\paragraph{Link 4: early peak, then decline}
Training then optimises against an almost-unfiltered, confidently-wrong
pseudo-label set from the very first epochs, before the teacher is good enough
to bootstrap. This is textbook confirmation bias~\cite{arazo2020pseudo}
compounded by the feature-distortion dynamic of
LP-FT~\cite{kumar2022finetuning}, whose finding that ``early stopping does not
mitigate feature distortion'' matches our curves precisely: at matched
batch~$16$ every adaptive run peaks within epochs~$4$--$20$ (the floor and
per-class rules within the first $4$) and then declines for the rest of
training, and at batch~$4$ the peak collapses to epochs~$1$--$5$
(\Cref{fig:trajectory}a). The strict threshold is the single rule that never
enters this chain, because its cutoff does not relax as the teacher becomes
confident; it keeps excising the error-enriched low-confidence band throughout.

\begin{table}[!tb]
\centering
\caption{\textbf{Mask flooding, measured.} (Matched batch~$16$, from the
TensorBoard logs.) ``first ep $\ge x$'' is the first epoch class-averaged
retention reaches $x$; ``$\rho$ final'' is the value at each run's last trained
epoch (runs early-stop at different epochs, Sec.~\ref{sec:negative}). The dynamic
rule floods (crosses $0.95$) by ep~$7$ and saturates at $1.000$; strict crosses
the same level three epochs later and stops at $0.989$, never reaching $0.99$;
the floor (\Cref{thm:floor}) never crosses $0.95$, the only adaptive rule
bounding retention away from~$1$. Note that strict too ends at high retention:
what distinguishes it is \emph{which} pixels it retains
(Sec.~\ref{sec:maskratio}), not how many.}
\label{tab:retention}
\footnotesize
\setlength{\tabcolsep}{3pt}
\resizebox{\columnwidth}{!}{%
\begin{tabular}{lccccc}
\toprule
Threshold rule & $\rho$ start & first ep $\ge0.90$ & first ep $\ge0.95$ & first ep $\ge0.99$ & $\rho$ final \\
\midrule
Strict fixed $\tau{=}0.95$ & ramp & 10 & 10 & --- & 0.989 \\
Dynamic (CW-BASS)          & 0.569 & 7 & 7 & 8 & 1.000 \\
Dynamic + floor + calib    & low   & 22\textsuperscript{\dag} & --- & --- & 0.914 \\
Per-class adaptive         & low   & 8 & --- & --- & 0.893 \\
\bottomrule
\end{tabular}}
\\[2pt]
{\footnotesize \textsuperscript{\dag}The floored run reaches only $\rho{=}0.914$
by epoch~$22$ and never crosses $0.95$, in agreement with
\Cref{thm:floor}.}
\end{table}

\begin{figure}[!htb]
\centering
\IfFileExists{figures/conf_hist.pdf}{%
  \includegraphics[width=0.75\linewidth]{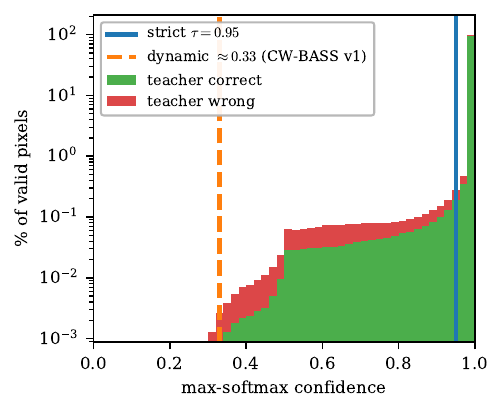}%
}{\fbox{\parbox{0.8\linewidth}{\centering\color{gray}\vspace{2cm}
  \textbf{Figure placeholder} — \texttt{scripts/plot\_confhist.py}.\vspace{2cm}}}}
\caption{\textbf{Confidence saturation: the chain's first link.} Per-pixel
max-softmax confidence of
the strict EMA teacher over Pascal~VOC val ($150$ images), split into pixels it
classifies correctly (green) vs incorrectly (red); log $y$-axis. $98\%$ of
pixels exceed $0.95$, and the teacher's errors concentrate in the
low/mid-confidence band that a strict cutoff rejects but a relaxed one keeps:
the dynamic threshold ($\approx0.33$) admits $99.9\%$ of all errors, the strict
$\tau{=}0.95$ only $63\%$. A threshold below the saturated mass keeps
everything, errors included; this is what disables every adaptive rule
downstream.}
\label{fig:confhist}
\end{figure}

\begin{figure}[!htb]
\centering
\IfFileExists{figures/mechanism.pdf}{%
  \includegraphics[width=0.66\linewidth]{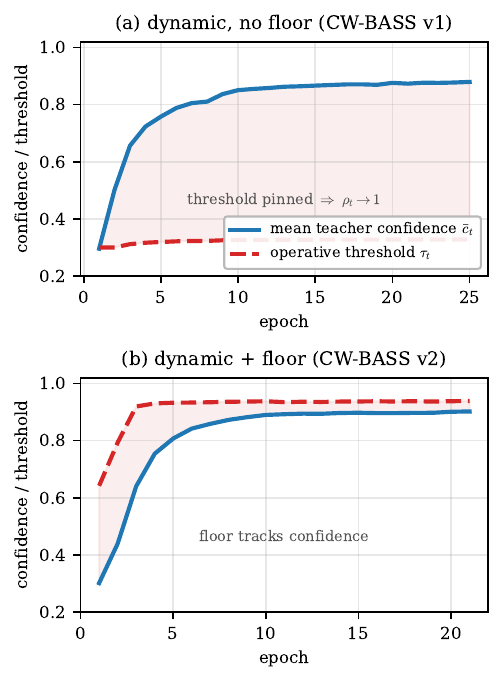}%
}{\fbox{\parbox{0.9\linewidth}{\centering\color{gray}\vspace{2cm}
  \textbf{Figure placeholder} — \texttt{scripts/plot\_mechanism.py}.\vspace{2cm}}}}
\caption{\textbf{Dynamic-range collapse, measured.} (Pascal~VOC 1/8, DINOv2-Base,
class-averaged.) \textbf{(a)} The bare dynamic threshold (red) stays pinned near
its lower clamp while the teacher's mean confidence (blue) climbs past $0.88$; the
shaded gap is confidence mass admitted indiscriminately (\Cref{cor:collapse}).
\textbf{(b)} The floor forces the operative threshold upward
($0.577\to0.922$, \Cref{thm:floor}) but only to a quantile still retaining
$\approx\!91\%$ of the mass, why it stabilises without recovering accuracy.}
\label{fig:mechanism}
\end{figure}

\paragraph{The per-class rule does not help the classes it targets}
The sharpest test of the per-class motivation is direct: does lowering a hard
class's threshold improve that class? \Cref{tab:tail} answers no. On the six
hardest Pascal classes, exactly those the rarity-scaled rule lowers thresholds
for, the per-class scheme merely \emph{ties} the global dynamic threshold at the
same batch~$4$ (tail mean $58.6$ vs $59.2$, inside seed noise): the targeted
coverage gain it is designed to deliver does not appear. This is the
batch-controlled core of the analysis: per-class buys nothing over the
simpler global rule, because at foundation strength a hard class's uncertainty
signals that the teacher is \emph{wrong}, not unfairly filtered
(Proposition~\ref{prop:unbiased}); relaxing the cutoff converts that uncertainty
straight into label noise. \Cref{fig:perclass} extends this to all $20$
foreground classes: strict (batch~$16$) matches or beats both adaptive rules
(batch~$4$) on \emph{every} one, the deficit widest on the cluttered indoor
classes (chair, diningtable, sofa, tvmonitor) where pseudo-label noise is most
damaging. That figure carries strict's batch advantage, so it shows the ceiling
rather than a matched contrast; at matched batch~$16$ and matched seed
(\Cref{tab:perclass_full}) strict's lead narrows to $18$ of $21$ classes, the
per-class rule taking \emph{bicycle}, \emph{train} and \emph{tvmonitor}.

\begin{figure*}[!tb]
\centering
\IfFileExists{figures/perclass_bars.pdf}{%
  \includegraphics[width=0.88\textwidth]{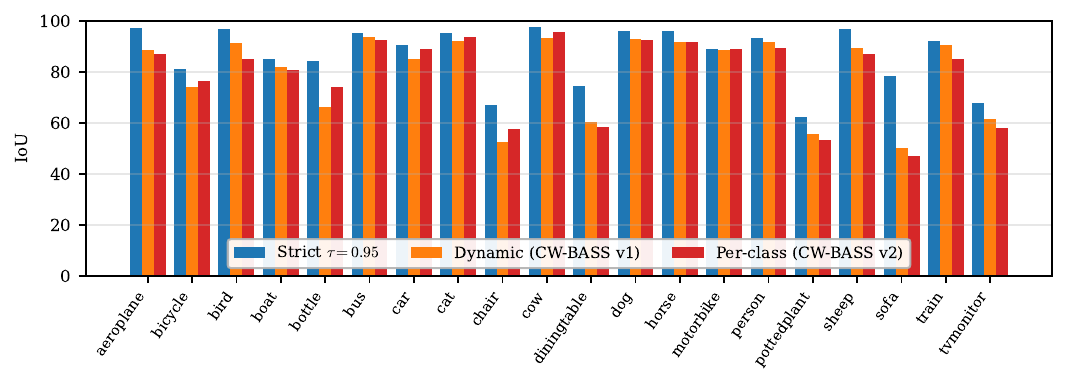}%
}{\fbox{\parbox{0.9\linewidth}{\centering\color{gray}\vspace{1.6cm}
  \textbf{Figure placeholder} — \texttt{scripts/plot\_mechanism.py}.\vspace{1.6cm}}}}
\caption{\textbf{Per-class IoU across all $20$ Pascal~VOC foreground classes.}
(1/8, DINOv2-Base, best EMA checkpoint.) The strict fixed threshold (blue, batch~$16$)
matches or exceeds the dynamic (orange) and per-class adaptive (red) rules
(both batch~$4$) on every class; the gap is largest on the hard indoor classes.
The batch-matched comparison here is dynamic vs.\ per-class: the two adaptive
rules are near-indistinguishable, underscoring that per-class adaptation adds
nothing over a global dynamic threshold at foundation strength; the strict
column carries a batch advantage and is included to show the ceiling, not as a
batch-matched contrast.}
\label{fig:perclass}
\end{figure*}

\begin{table}[t]
\centering
\caption{\textbf{Per-class IoU on the six hardest Pascal~VOC classes.} (1/8,
DINOv2-Base, best EMA.) The per-class rule lowers thresholds for exactly these classes, yet at
matched batch~$4$ it merely \emph{ties} the global dynamic threshold (tail mean
$58.6$ vs $59.2$, within seed noise): the targeted coverage gain does not appear.
Batch-matched comparison is \emph{Dynamic} vs \emph{Per-class} (both b4);
\emph{Strict} (b16) is shown for reference only.}
\label{tab:tail}
\footnotesize
\begin{tabular}{lccc}
\toprule
Class & Strict (b16) & Dynamic (b4) & Per-class (b4) \\
\midrule
bicycle      & $\mathbf{81.3}$ & 74.4 & 76.5 \\
chair        & $\mathbf{67.0}$ & 52.5 & 57.7 \\
diningtable  & $\mathbf{74.6}$ & 60.6 & 58.5 \\
pottedplant  & $\mathbf{62.4}$ & 55.7 & 53.5 \\
sofa         & $\mathbf{78.4}$ & 50.4 & 47.2 \\
tvmonitor    & $\mathbf{67.8}$ & 61.6 & 58.1 \\
\midrule
tail mean    & $\mathbf{71.9}$ & 59.2 & 58.6 \\
overall mIoU & $\mathbf{87.4}$ & 80.4 & 80.0 \\
\bottomrule
\end{tabular}
\end{table}

\Cref{tab:perclass_full} extends this to all $20$ foreground classes plus
background at matched batch~$16$ and matched seed: the deployed strict rule leads
in mean IoU and on $18$ of the $21$ classes. The per-class adaptive rule wins only
on \emph{bicycle} ($+0.5$), \emph{train} ($+1.6$) and \emph{tvmonitor} ($+4.3$),
and pays heavily where a lowered threshold admits confidently-wrong pixels in
cluttered indoor scenes (\emph{sofa} $-21.0$, \emph{chair} $-17.8$,
\emph{bottle} $-5.6$).

\begin{table*}[!tb]
\centering
\caption{\textbf{Per-class IoU on Pascal~VOC 1/8} (DINOv2-Base, best EMA
teacher, matched batch~$16$, \emph{both rows seed~$0$}). Deployed strict rule
vs.\ the per-class adaptive rule; strict leads in mean IoU and on $18$ of the
$21$ classes.}
\label{tab:perclass_full}
\scriptsize
\setlength{\tabcolsep}{3pt}
\resizebox{\textwidth}{!}{%
\begin{tabular}{l c c *{21}{c}}
\toprule
Rule & mIoU & & bg & aero & bike & bird & boat & bottle & bus & car & cat & chair & cow & table & dog & horse & mbike & person & plant & sheep & sofa & train & tv \\
\midrule
Strict $\tau{=}0.95$ (deployed) & \textbf{87.4} & & 96.8 & 97.2 & 81.3 & 97.0 & 85.1 & 84.3 & 95.4 & 90.7 & 95.4 & 67.0 & 97.6 & 74.6 & 96.2 & 96.3 & 89.0 & 93.6 & 62.4 & 97.1 & 78.4 & 92.2 & 67.8 \\
Per-class adaptive & 84.1 & & 96.1 & 95.0 & 81.8 & 94.6 & 84.4 & 78.8 & 93.0 & 87.6 & 94.4 & 49.2 & 95.5 & 71.6 & 93.6 & 95.0 & 84.8 & 92.0 & 61.9 & 92.8 & 57.4 & 93.8 & 72.1 \\
\bottomrule
\end{tabular}}
\end{table*}

\subsection{Mask-Ratio Diagnostic}
\label{sec:maskratio}

The full-scale trajectories (\Cref{fig:trajectory}) make the mechanism legible
and confirm \Cref{thm:floor}; read the retention panel~(b) alongside the
accuracy panel~(a). The \emph{dynamic} and \emph{per-class} rules drive
retention $\rho_t$ to $\approx\!1$ within a few epochs (\Cref{tab:retention}),
\emph{premature mask flooding}: they admit essentially every pixel \emph{before}
the teacher is accurate (\Cref{cor:collapse}), so the student trains on a
saturated, still-wrong pseudo-label set from the start and the accuracy panel
shows it collapsing right after its early peak. The \emph{floor} instead holds
$\rho_t$ bounded (never crossing $0.95$), in agreement with the fixed quantile
\Cref{thm:floor} predicts; the realized ceiling ($\approx\!0.91$) sits a little
below the nominal $s{=}0.95$ because the operative threshold is
$\max(\tau^{\mathrm{dyn}},\tau^{\mathrm{floor}})$. That bound \emph{reduces} the
late-training collapse but does not eliminate it (\Cref{tab:negative}): as
panel~(a) shows, the floor-stabilised run still tracks the other adaptive
variants down. Bounding retention does not undo the damage of having relaxed the
cutoff.

Strict is the instructive contrast, and it is worth being exact about what the
contrast is, because strict also ends at high retention ($\rho{=}0.989$): high
retention is plainly not \emph{sufficient} for collapse. Two differences remain.
Strict reaches the flood later (first crossing $0.90$ and $0.95$ at epoch~$10$,
against epoch~$7$ for dynamic) and, more importantly, it reaches it by a
different route: strict's cutoff never moves, so its retention can only rise
because the teacher's confidence rose, i.e.\ its high retention is a
\emph{consequence} of the teacher improving, whereas the dynamic rule's is a
consequence of the cutoff sitting beneath the mass from the first epoch. The
three-epoch difference is small, and we do not claim the timing alone explains
the outcome. What the trajectories do establish is the conjunction: the rules
that flood do so while their own cutoff is pinned near $0.33$, admitting
$99.9\%$ of teacher errors throughout, whereas strict admits $63\%$ at every
point on its ramp. It is the composition of retained pixels, not the retained
\emph{fraction}, that separates them.

\begin{figure*}[!tb]
\centering
\IfFileExists{figures/trajectory.pdf}{%
  \includegraphics[width=0.88\linewidth]{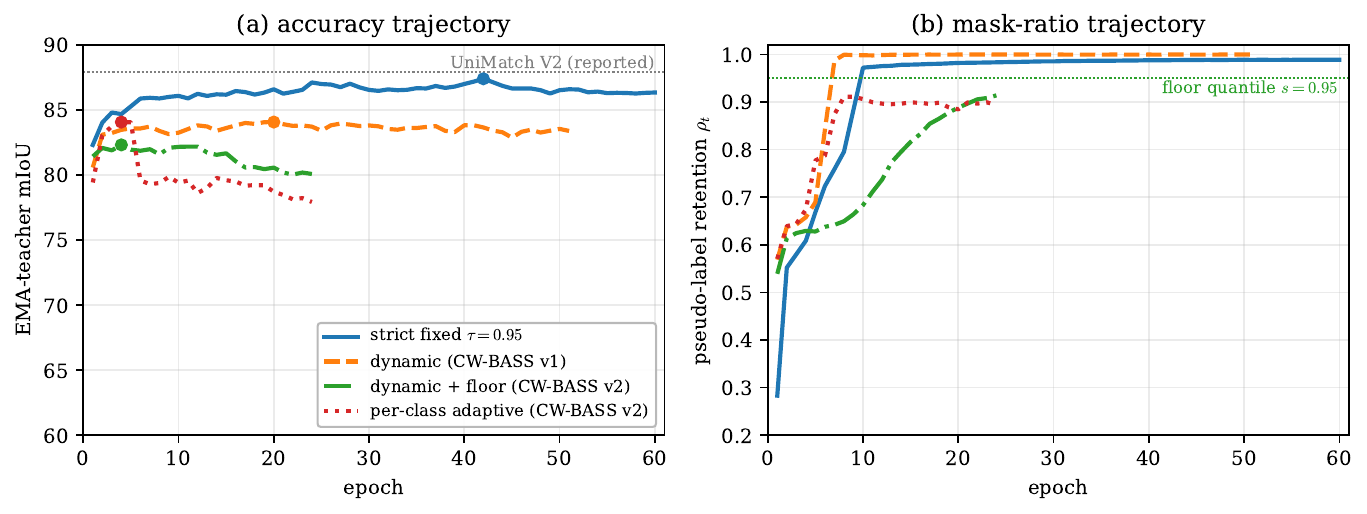}%
}{%
  \fbox{\parbox{0.9\linewidth}{\centering\color{gray}\vspace{2.2cm}
  \textbf{Figure placeholder} — generate with
  \texttt{python -m scripts.plot\_trajectories}.\vspace{2.2cm}}}%
}
\caption{\textbf{The saturation analysis in one figure: strict climbs, adaptive
peaks early.} Matched-batch trajectories on Pascal~VOC 1/8
(DINOv2-Base, all four rules at batch~$16$).
\textbf{(a)} EMA-teacher mIoU: strict climbs to $87.4$ and holds; every adaptive
rule peaks early (ep~$4$--$20$). What follows the peak differs by rule and we do
not lump them together: the per-class rule declines steeply ($-6.1$~mIoU EMA),
the floor moderately ($-2.2$), while the dynamic rule is essentially flat
($-0.7$) --- it fails by never climbing, not by falling.
\textbf{(b)} Retention $\rho_t$: the dynamic and per-class
rules flood toward $\approx\!1$ (dynamic crosses $0.95$ by ep~$7$,
\Cref{tab:retention}); the floor stays bounded ($\approx\!0.91$, \Cref{thm:floor});
strict reaches a comparable $0.989$ but three epochs later and while rejecting
the error-enriched band throughout (Sec.~\ref{sec:maskratio}). Curves end at
different epochs because training used early stopping (patience~$20$) inside the
$60$-epoch budget: strict ran to ep~$60$, the dynamic rule stopped at $51$, the
floor and per-class rules at $24$.}
\label{fig:trajectory}
\label{fig:maskratio}
\end{figure*}

\subsection{Floor Stability and Calibration Cost}
\label{sec:analysis}

The floor's role is narrow and important: it confirms the collapse
\emph{mechanism} by removing it, not the accuracy deficit. In a three-seed,
reduced-scale study (supplementary material, \Cref{tab:stability}), the bare
dynamic threshold degrades sharply after an early peak (mean drop $15.8$~mIoU,
seed spread $\pm7.3$); the same threshold lower-bounded by the floor
(\Cref{thm:floor}) arrests most of that degradation (mean drop $1.7$) and
collapses the seed spread to $\pm0.6$. This confirms \Cref{thm:floor,cor:collapse}:
the bare dynamic threshold is not merely suboptimal but \emph{unstable} once the
teacher saturates, and pinning retention to a fixed quantile restores stability.
\emph{But stability is not accuracy}: both rules sit far below the strict
threshold (\Cref{tab:negative}). The floor cures the \emph{instability} of a
confidence-blind rule but not its \emph{mismatch} to a saturated teacher: the
honest lesson is that one should not relax the threshold in the first place, and
the floor is the instrument that proves \emph{why} relaxing it does the damage.
A calibration-fraction sweep (supplementary material, \Cref{tab:alpha}) shows
accuracy is flat in $\alpha$, all four settings within $0.5$~mIoU, so holding out
a small labeled slice costs essentially nothing; the slice's value is not raw
accuracy but the \emph{unbiased} noise estimate of \Cref{prop:unbiased}, which
removes the self-confirming in-batch estimate by which falling noise could have
justified lowering thresholds, and whose prediction (that $\widehat\varepsilon_k$
does not fall where the rules lower thresholds) is supported by the per-class IoU
evidence (\Cref{tab:tail}).

\subsection{What This Means for the Method}
This synthesis closes the analysis: it is why CW-BASS~v2 \emph{gates} its adaptive
machinery rather than deploying it unconditionally. On a saturated teacher (the
Pascal regime, $98\%{\ge}0.95$, and the standard foundation-model benchmarks) the
method selects the strict fixed
threshold; it engages the floor only where it pays off, on a confidently-unreliable
teacher whose $\pi_{\mathrm{kept}}$ falls below the confidence it demands
(Sec.~\ref{sec:generality}). The floor and the
held-out calibration are not discarded under saturation, they remain the
instruments that make the gate's call legible: the floor as a provable stability
guarantee that isolates the collapse mechanism (Sec.~\ref{sec:maskratio}), and the
calibration as the unbiased diagnostic that reads the teacher's noise
(Sec.~\ref{sec:calibration}). The broader lesson the controlled comparison
establishes is that one should not spend effort re-tuning dynamic or per-class
thresholds for a \emph{saturated} teacher: the room they leave to adapt downward is
room to admit the error-enriched band beneath the saturated mass.
\section{Discussion}
\label{sec:discussion}

CW-BASS~v2 is a saturation-aware selection method, and its design is dictated by a
measured mechanism rather than a heuristic. An entire family of pseudo-label
selection rules (dynamic global thresholds, per-class curricula, soft confidence
weights) was developed for weak, under-confident ResNet teachers. On a strong,
\emph{saturated} DINOv2 teacher that family does not merely fail to help; it
\emph{actively harms}: every adaptive variant peaks within a few epochs and then
degrades under its own pseudo-label noise, while the strict fixed threshold climbs
cleanly ahead. Where such a teacher is genuinely uncertain it is uncertain because
it is \emph{wrong}, so any rule that \emph{lowers} a cutoff in response to
confidence buys coverage in pure noise, the inverse of the regime FlexMatch and
FreeMatch were designed for. The method's answer is not to abandon adaptive
filtering but to \emph{gate} it: measure the reliability of the teacher's confident
set, then use the rule that regime warrants.

\paragraph{Why the field's adaptive-threshold intuition inverts at foundation
strength.}
This \emph{regime inversion} is the paper's central conceptual claim. FlexMatch
and FreeMatch derive their per-class thresholds from a \emph{per-class
learning-status signal} (a count of confident predictions, or a per-class
confidence EMA relative to the global mean), and the mechanism is informative
only to the extent that these signals \emph{differ across classes}: an easy
class sits high, a hard class low, and the rule lowers the cutoff for the
laggard. That spread is precisely what a weak ResNet teacher supplies and what a
DINOv2 teacher destroys: with nearly all confidence squeezed against the ceiling,
every per-class statistic collapses toward the same value (our per-class cutoff
rises only to $0.805$, the two adaptive rules becoming near-indistinguishable,
\Cref{fig:perclass}). A mechanism whose entire value is discriminative power
across classes loses that power exactly when the backbone is strongest. The
FlexMatch/FreeMatch family is thus not wrong; it is correctly tuned for a noise
regime that foundation backbones have left behind.

\paragraph{What the method keeps, and why}
CW-BASS~v2's two new constructs serve double duty. The self-adaptive floor is the
adaptive rule the gate engages under a confidently-unreliable teacher, where it
\emph{wins} (ADE20K, \Cref{tab:datasets}); it is also a provable stability
guarantee (\Cref{thm:floor}) that let us show the saturated-teacher collapse is
causal by removing it. The held-out calibration is an unbiased per-class noise
estimator (\Cref{prop:unbiased}) that both breaks the self-confirming loop of
in-batch estimation and gives the gate an honest read of the teacher's regime. The
resulting rule is conditional: strict filtering \emph{when the confident set is
reliable} ($\pi_{\mathrm{kept}}{\ge}\tau$), the adaptive floor when it is not
(Sec.~\ref{sec:generality}). The complementary question, whether
\emph{embedding}-space auxiliaries rather than threshold engineering can add to a
foundation-model SSSS baseline, is taken up by PixCon~\cite{pixcon} and is out of
scope here.

\paragraph{Relation to other audits}
Our analysis sits in the lineage of realistic-evaluation
work~\cite{oliver2018realistic} and the recent audit of UniMatch~V2
itself~\cite{landgraf2025rethinking}. Landgraf \emph{et al.} probe that system's
\emph{reliability and robustness} (calibration, out-of-distribution behaviour,
and sensitivity to design choices), holding the method fixed and varying the
evaluation. We do the complementary thing: hold the backbone, decoder, optimiser,
augmentation, crop, batch, splits and epoch budget fixed and vary the threshold
rule, at matched batch and over three seeds, so that the mechanism
(dynamic-range collapse $\to$ mask flooding) and the deployable decision it
implies (the reliability gate) are attributable to the selection rule. Within the
adaptive family the rule is the \emph{sole} varied factor; against the strict
arm the comparison is recipe-level, and Sec.~\ref{sec:negative-confound} bounds
what that costs. Neither varies the threshold rule as the
principal factor; to our knowledge ours is the first batch-matched,
mechanism-level study of adaptive pseudo-label thresholding on a
foundation-model backbone. It also refines the calibration
view~\cite{guo2017cal} in an important way. The usual concern is that strong
models are globally mis-calibrated; the DINOv2 teacher is not, in aggregate
(\Cref{fig:calibration}), yet confidence-adaptive selection still fails. The
operative pathology is not aggregate miscalibration but \emph{confidence
saturation}, the collapse of the confidence distribution's dynamic range, which
disables any rule that reads confidence to decide what to keep. What little
spread survives is, moreover, the region where the teacher is over-confident
(\Cref{fig:calibration}b), so relaxing the cutoff to exploit that spread recovers
mostly errors: the two observations are complementary, not in tension.

\subsection{Implications for Practice}
\label{sec:implications}
The gate, and the analysis behind it, reduce to four cheap checks.
Each is instantiated in this paper, each is cheap relative to a single training
run, and the first costs one forward pass.

\keybox{\textbf{Before proposing or adopting an adaptive pseudo-label
threshold at foundation scale:}
\begin{enumerate}[leftmargin=1.6em,itemsep=1pt,topsep=3pt]
  \item \textbf{Measure confident-set reliability, not just saturation.} On a
  held-out labeled slice, report $\pi_{\mathrm{kept}}{=}\Pr[\text{correct}\mid
  c{\ge}\tau]$ (one forward pass). If it meets the confidence demanded
  ($\pi_{\mathrm{kept}}{\ge}\tau$) the confident set is trustworthy and a strict
  cutoff is right; if not (ADE20K, $89\%$) the teacher is confidently
  miscalibrated and a softened rule helps. Saturation alone is \emph{not} enough:
  all our teachers are ${\ge}82\%$ saturated (\Cref{tab:gate}) yet the verdict
  flips with $\pi_{\mathrm{kept}}$.
  \item \textbf{Match the batch.} Compare threshold rules at identical batch
  size from one shared config (\Cref{tab:negative}, top block). Batch size
  modulates the severity of a collapse and can masquerade as a threshold
  effect.
  \item \textbf{Report trajectories and final-vs-best, not best
  checkpoints.} The early-peak-then-decline signature is invisible in a
  best-checkpoint table (\Cref{fig:trajectory}).
  \item \textbf{Check the best-vs-final gap, on the EMA teacher itself.} A
  best-checkpoint number hides the collapse whether or not the EMA lags: here the
  per-class rule's best EMA of $84.07$ conceals a final EMA of $77.93$
  (\Cref{tab:negative}), a $6.14$ fall larger than its student's.
\end{enumerate}}

\paragraph{Scope and limitations}
\label{sec:limitations}
A method's claims are only as good as their stated scope, so we state ours exactly.
The controlled comparison is full-scale and \emph{three-seed} on Pascal, strict
included (\Cref{tab:multiseed}), and the strict-vs-adaptive verdict reproduces
across the DINOv2-S/B/L scale family and Cityscapes (\Cref{tab:datasets}). Four
limits bound the generality. \emph{(i)~Backbone family.} Every teacher we test is
DINOv2; $\pi_{\mathrm{kept}}$ is a calibration property that ought to transfer to
other foundation teachers, but we demonstrate the gate only for
\emph{DINOv2-family} teachers and cannot claim it for CLIP- or SAM-style encoders
without running them. \emph{(ii)~The affirmative result is single-seed.} The one
regime where the adaptive floor wins is the confidently-unreliable ADE20K teacher;
its \emph{sign} (no collapse, floor${\ge}$strict) reproduces at S and B, but the
$+1.5$ magnitude is within plausible seed noise, so we do not rest the paper on it.
\emph{(iii)~The gate boundary is validated, not calibrated, and validated
post-hoc.} With six teachers we show $\pi_{\mathrm{kept}}{\ge}\tau$ separates the
regimes and makes the right strict-vs-floor call blind
(\Cref{fig:gate,tab:gate}); pinning the exact boundary would need many more
teachers than compute allowed. The validation also measures $\pi_{\mathrm{kept}}$
on each dataset's \emph{validation} split from a converged teacher, not on
$\mathcal{L}_{\mathrm{cal}}$ mid-run as a deployment would (Sec.~\ref{sec:gate}),
because a $9$-image calibration slice cannot resolve $98\%$ from $89\%$; we have
therefore validated the criterion, not measured the live system. \emph{(iv)~Strict's own seed variance.} One of
three strict seeds stalled at $84.09$, so strict's \emph{margin} over the closest
adaptive rule is modest and, at three seeds, significant only over the floor and
FreeMatch; on that seed the dynamic rule finishes above strict, the one cell in
the grid where an adaptive rule wins. What is unambiguous is that strict uniquely
reaches the ${\sim}87.4$ mode no adaptive run attains.
\emph{(v)~The strict arm is not loss-matched.} Strict runs the UniMatch~V2
unlabeled loss while the adaptive arms run the CW-BASS~v2 one, and two adaptive
rules additionally hold out $5\%$ of labels
(Sec.~\ref{sec:whatiscontrolled}). The measured components of that difference
account for ${\approx}1.2$ of the $3.3$--$5.1$~mIoU gap
(Sec.~\ref{sec:negative-confound}) and the dual-strong-view difference is not
ablated at all; a strict arm run \emph{inside} the CW-BASS~v2 loop would make the
comparison single-factor and is the most valuable single experiment left
undone. Two smaller caveats: the DINOv2-L cross-dataset
check is partial (strict L-ADE exceeded GPU memory, L-Cityscapes not run), and the
batch-$4$ blocks of \Cref{tab:negative} mix batch sizes and were cut at a compute
deadline, so we treat the batch-severity trend as \emph{qualitative} and rest the
core claim on the matched-batch~$16$ comparison (Sec.~\ref{sec:negative-confound}).
The bounded-retention theorem rests on a scale-family idealisation
(Assumption~\ref{ass:scale}); its operative content is the empirical mask-ratio
prediction. The calibration slice costs $5\%$ of labels, non-trivial at very low
budgets; we have not optimised $\alpha$. Finally, we audit the CAFS/ENCORE
\emph{mechanism} (held-out and feedback-driven per-class thresholds) through our
own faithful reimplementation (Sec.~\ref{sec:perclass}) rather than their released
code; a drop-in run of the published systems is future work.

\takeaway{The regime inverted because the signal every adaptive rule
reads, confidence spread, no longer exists at foundation strength, however
well-calibrated the teacher. The four checks above are the gate, distilled;
run them before adapting.}

\section{Conclusion}

We presented CW-BASS~v2, a saturation-aware pseudo-label selection method for
semi-supervised segmentation under foundation-model teachers, together with the
controlled analysis that makes it principled. Its premise is that the right
threshold rule depends on the reliability of the teacher's confident set, which the
method measures one-pass on a held-out slice and gates on. On an identical
DINOv2-Base backbone the dynamic, per-class-adaptive and floor-stabilised
variants --- which differ from one another in the rule alone --- all peak early and top out at
$82$--$84$~mIoU at the matched batch~$16$ (collapsing to $\approx\!80$ at small
batch) on Pascal~VOC 1/8, while the strict fixed threshold reaches $87.4$, a
single-seed reproduction of UniMatch~V2's reported $87.9$; run unconditionally on
this reliable, saturated teacher the floor we ourselves proposed fares worst, so
CW-BASS~v2 selects strict here. The mechanism is measured at every link: confidence
saturation collapses the confidence distribution's dynamic range, the adaptive
cutoff degenerates to a constant exactly where the theory says it must, the
retention mask floods, and self-training turns into confirmation bias.

Direct measurement of six DINOv2 teachers sharpens the generality claim and
corrects a tempting story. All six are highly saturated, so ``under-confidence''
does not explain when adaptive filtering helps; what does is the reliability of the
confident set, $\pi_{\mathrm{kept}}$: it is ${\approx}98\%$ on Pascal and
Cityscapes (where strict wins or ties) and ${\approx}89\%$ on ADE20K, the one
teacher where the floor edges ahead (single seed). The gate reads exactly this
quantity, with the operating threshold as its boundary, and picks the
correct strict-vs-floor call on all six teachers blind. The practical message is therefore
not ``always strict'' but a \emph{decision rule} CW-BASS~v2 embodies: measure
whether the confident set earns the confidence you demand, filter strictly when it
does, soften when it does not. The methodological counterpart is shorter still:
measure before you adapt.

\paragraph{Reproducibility}
All code, configs, checkpoints and the experiment-matrix runner are released at
\url{https://github.com/psychofict/CW-BASS-v2}, together with the scripts that
regenerate every figure and table in this paper from the released run artefacts
(\texttt{scripts/plot\_*.py}, \texttt{scripts/gate\_stat.py}). The DINOv2 runs
follow the public UniMatch~V2 protocol so that the baseline comparison is
protocol-faithful.

\bibliographystyle{IEEEtran}
\bibliography{references}

\appendix
\section{Supplementary Material}
\label{sec:supplementary}
Two analyses that Sec.~\ref{sec:negative} summarises in the main text are
reported here in full: a threshold-parameter sweep ruling out a
stale-constants explanation of the dynamic rule's collapse, and the
reduced-scale, three-seed floor-stability and calibration-cost studies behind
Sec.~\ref{sec:analysis}'s numbers.

\subsection{Batch-4 Runs}
\label{sec:supp-b4}
The early-peak signature recurs, more violently, at batch~$4$: the three adaptive
rules peak within $1$--$5$ epochs below $80.5$ and the raw student then falls
steeply (\Cref{tab:negative_b4}). We read these runs \emph{qualitatively} only.
They mix batch sizes relative to the strict baseline, which was run only at
batch~$16$, and they were cut at a compute deadline, so their raw drop
\emph{understates} the full collapse and the EMA-drop column is left blank. We
claim from them just the \emph{direction} --- adaptive rules degrade harder when
each noisy update carries more weight --- and no numeric severity gradient.

\begin{table}[h]
\centering
\caption{\textbf{Batch-4 runs (qualitative only).} Pascal~VOC, DINOv2-Base,
single seed. ``best EMA (ep)'' is the best EMA mIoU and its epoch; ``raw drop''
the student's peak-to-final fall. Cut at a compute deadline, so no EMA drop is
reported and the raw drop understates the collapse.}
\label{tab:negative_b4}
\footnotesize
\setlength{\tabcolsep}{2.4pt}
\begin{tabular}{llccc}
\toprule
Threshold rule & batch & best EMA (ep) & EMA drop & raw drop \\
\midrule
\multicolumn{5}{l}{\emph{Pascal~VOC 1/8 (183 labels):}}\\
Dynamic (CW-BASS)        & 4 & $80.40$ (5) & --- & $2.2$ \\
Dynamic + floor + calib  & 4 & $80.50$ (1) & --- & $5.3$ \\
Per-class adaptive       & 4 & $79.96$ (2) & --- & $12.6$ \\
\midrule
\multicolumn{5}{l}{\emph{Pascal~VOC 1/4 (366 labels):}}\\
Dynamic (CW-BASS)        & 4 & $80.55$ (1) & --- & $2.2$ \\
Dynamic + floor + calib  & 4 & $\mathbf{82.95}$ (4) & --- & $4.2$ \\
\bottomrule
\end{tabular}
\end{table}

\subsection{Threshold-Parameter Sweep: Is It the Stale Constants?}
\label{sec:taus}
A natural objection is that the dynamic rule loses only because it is run at its
original ResNet-era constants ($\tau_0{=}0.6$, $\beta{=}0.5$, $\tau_{\min}{=}0.3$),
which by construction ceiling its cutoff near $0.34$ (\Cref{cor:collapse}); perhaps a
re-tuned dynamic rule would close the gap. We test this directly by sweeping the two
knobs that lift the ceiling (the lower clamp $\tau_{\min}$ and the base $\tau_0$) on
Pascal~VOC 1/8 at batch~$16$ (\Cref{tab:taus}). It does not close the gap. Raising
$\tau_{\min}$ through $0.5$, $0.7$, $0.9$ leaves the rule in the same $81$--$85$ band as
the original $0.3$ setting, still peaking early (epochs $2$--$9$); raising $\tau_0$ to
$1.2$/$1.7$ behaves identically. The best cell in the sweep ($84.67$ at
$\tau_{\min}{=}0.5$) is still $2.7$~mIoU short of strict. The only setting that would
close the gap is $\tau_{\min}{=}0.95$, where the lower clamp is pinned at the strict
threshold and the rule has degenerated into the strict rule itself with no adaptive
range left; that is an algebraic identity, not a swept cell, and we do not report it as
a measurement. The deficit is therefore \emph{not} an artefact of a strawman
parameterisation: across the whole sweep no setting that retains adaptive range
approaches strict.

\begin{table}[h]
\centering
\caption{\textbf{The gap is not a stale-hyperparameter artefact.} Sweeping the
dynamic rule's ceiling knobs (lower clamp $\tau_{\min}$; base $\tau_0$) on
Pascal~VOC 1/8 at batch~$16$ (single seed). Every setting stays in the
$81$--$85$ band and peaks early; the best is still $2.7$~mIoU short of strict's
$87.40$. Re-tuning the adaptive rule upward does not recover strict's
performance. The limiting case $\tau_{\min}{=}0.95$ is omitted deliberately: it
pins the cutoff at the strict threshold, so the rule \emph{is} strict by
construction, and we do not report an identity as a swept measurement. Note the
$\tau_0{=}1.7$ row: its cutoff approaches the strict value \emph{late} in
training yet the run is the sweep's worst, because early on, when $\bar c$ is
still low, the rule admits the error-enriched band and the damage is already
done.}
\label{tab:taus}
\footnotesize
\setlength{\tabcolsep}{5pt}
\begin{tabular}{lcc}
\toprule
Dynamic-rule setting & best EMA (ep) & note \\
\midrule
$\tau_{\min}{=}0.30$ (original) & $84.07$ (20) & ceiling ${\approx}0.34$ \\
$\tau_{\min}{=}0.50$            & $\mathbf{84.67}$ (7)  & sweep best; $-2.73$ vs strict \\
$\tau_{\min}{=}0.70$            & $82.73$ (2)  & \\
$\tau_{\min}{=}0.90$            & $83.68$ (6)  & \\
$\tau_0{=}1.2$ (raised ceiling) & $83.77$ (9)  & cutoff ${\to}0.67$ \\
$\tau_0{=}1.7$ (raised ceiling) & $81.61$ (6)  & cutoff ${\to}0.95$ (clipped) \\
\midrule
\multicolumn{3}{l}{\emph{Reference:} strict fixed $\tau{=}0.95$ \quad $87.40$ (42)} \\
\bottomrule
\end{tabular}
\end{table}

\subsection{Reduced-Scale Floor-Stability and Calibration-Cost Studies}
\label{sec:supp-stability}
\Cref{tab:stability} reports peak vs.\ final-epoch EMA mIoU over \emph{three
seeds} in a controlled reduced-scale study, comparing the bare dynamic
threshold against the same threshold lower-bounded by the floor. Without the
floor, retention collapses to $\approx\!1$ within six epochs, flooding the loss
with pseudo-label noise, and the EMA teacher \emph{degrades sharply} after an
early peak (mean drop $15.8$~mIoU across three seeds). With the floor
(\Cref{thm:floor}) retention stays bounded and the degradation is largely
arrested (mean drop $1.7$, seed variance $\pm7.3\!\to\!\pm0.6$). This is a
within-family comparison only: both rows still sit far below the strict
threshold (\Cref{tab:negative}).

\begin{table}[h]
\centering
\caption{\textbf{Within-family stability study.} (Pascal~VOC, DINOv2-Base.) The floor
removes the late-training collapse of the bare dynamic threshold. Peak vs.\
final-epoch EMA mIoU; a smaller drop is better. Reduced-scale validation
(crop~518, batch~2, 15 epochs, classic labeled split $+$ 1200-image unlabeled
subset). Note that \emph{both} rows sit far below the strict fixed threshold
($87.4$, \Cref{tab:negative}): the floor buys stability within the adaptive
family, not parity with the strict baseline.}
\label{tab:stability}
\footnotesize
\begin{tabular}{lccc}
\toprule
Threshold rule & peak mIoU & final mIoU & drop \\
\midrule
\multicolumn{4}{l}{\emph{1/8 split (3 seeds, mean$\pm$std):}}\\
Dynamic, no floor   & $72.3\pm7.3$ & $56.5$ & $15.8$ \\
Dynamic + floor     & $\mathbf{78.5\pm0.6}$ & $\mathbf{76.8}$ & $\mathbf{1.7}$ \\
\midrule
\multicolumn{4}{l}{\emph{1/4 split (1 seed):}}\\
Dynamic, no floor   & $77.7$ & $67.8$ & $9.9$ \\
Dynamic + floor     & $\mathbf{82.2}$ & $\mathbf{82.2}$ & $\mathbf{0.0}$ \\
\bottomrule
\end{tabular}
\end{table}

\paragraph{Reduced-scale validation protocol}
The three-seed stability study (\Cref{tab:stability}) was run on a single
8\,GB GPU at reduced scale (DINOv2-Base, crop~518, batch~2, 15 epochs, the
classic 1/8 labeled set of $183$ images with a $1200$-image unlabeled subset,
evaluated on a $300$-image validation subset) to obtain seed variance cheaply;
the full-scale single-seed trajectories of \Cref{fig:trajectory} use the
benchmark protocol of \Cref{tab:negative}. These reduced settings are
sufficient to exhibit the retention dynamics the theory concerns (the teacher
saturates and the bare threshold collapses within a few epochs). We report them
because the collapse-vs-bounded contrast is exactly what \Cref{thm:floor}
predicts and is robust to scale; the full-scale, strict-inclusive multi-seed
evidence lives in \Cref{tab:multiseed}, so this study serves as a focused,
cheap probe of the floor's stability effect rather than as the paper's only
multi-seed data.

\Cref{tab:alpha} sweeps the calibration fraction $\alpha$ (floor active
throughout): accuracy is flat, all four settings within $0.5$~mIoU, so holding
out a small labeled slice costs essentially nothing.

\begin{table}[h]
\centering
\caption{\textbf{Calibration-fraction sweep on Pascal~VOC 1/8.} (DINOv2-Base, reduced
scale, cwbass\_v2 seed~0.) $\alpha{=}0$ disables held-out calibration (the
in-batch estimator). Accuracy is flat in $\alpha$ (all settings within
$0.5$~mIoU): holding out the slice costs nothing, and its value is the
unbiased noise estimate of \Cref{prop:unbiased}, not mIoU.}
\label{tab:alpha}
\footnotesize
\begin{tabular}{cc}
\toprule
$\alpha$ (calib.\ fraction) & best EMA mIoU \\
\midrule
0.00 & 78.76 \\
0.02 & 78.49 \\
0.05 & 78.28 \\
0.10 & 78.51 \\
\bottomrule
\end{tabular}
\end{table}

\end{document}